\documentclass[preprint,12pt]{elsarticle}

\usepackage{amssymb}
\usepackage{amsmath}
\usepackage{graphicx}
\usepackage{color}
\usepackage{subfig}

\newtheorem{definition}{Definition} 
\newtheorem{observation}{Observation} 
\newtheorem{property}{Property}[section]

\newcounter{myexamplec}
\newenvironment{myexample}[1]{\refstepcounter{myexamplec} \medskip 
\noindent {\bf Example \themyexamplec~(#1).}~}{\hfill$\Box$\medskip}
\newenvironment{exampleContinued}[2]{\medskip\noindent {\bf Example #1~(#2\itshape{ -- continued}\,\bf).}~}{\hfill$\Box$\medskip}

\newtheorem{theorem}{Theorem}[section]
\newtheorem{lemma}{Lemma}[section]
\newtheorem{proposition}{Proposition}[section]

\newenvironment{proof}{\begin{trivlist}\item[\hskip\labelsep{\it Proof.}]}{\hfill$\Box$\end{trivlist}}
\newenvironment{proofOf}[1]{\begin{trivlist}\item[\hskip\labelsep{\it Proof of #1.}]}{\hfill$\Box$\end{trivlist}}
\newtheorem{claim}{Claim}

\newcommand{\C}{{\mathcal{C}}}
\newcommand{\E}{{\mathcal{E}}}
\newcommand{\G}{{\mathcal{G}}}

\newcommand{\M}{{\cal M}}
\newcommand{\N}{{\cal N}}
\newcommand{\MM}{{\rm \cal MM}}

\newcommand{\R}{{\cal R}}
\renewcommand{\S}{{\mathcal{S}t}}
\newcommand{\coNP}{{\rm coNP}}
\newcommand{\NP}{{\rm NP}}
\newcommand{\Pol}{{\rm P}}

\renewcommand{\P}{\Pi}
\newcommand{\Pnd}{\P^{nd}}

\newcommand{\atom}[1]{{\it atom}(#1)}
\newcommand\sel{\textrm{super-elementary}}
\newcommand\Sel{\mbox{Super-elementary}}
\newcommand\steady{{\rm steady}}
\newcommand\eG{\widehat{\G}}

\newcommand{\simpl}[2]{\sigma_{#2}(#1)}
\newcommand{\ssimpl}[2]{\Sigma_{#2}(#1)}
\newcommand{\ssimplnd}[2]{\Sigma^{nd}_{#2}(#1)}

\newcommand\nbd{{\rm not}~}

\usepackage{algorithmic}
\usepackage[ruled,vlined,linesnumbered,nokwfunc]{algorithm2e}

\SetNlSty{textrm}{}{:}
\SetAlgoNlRelativeSize{-1}
\SetKwComment{tcp}{{$/\!/$}}{}
\SetCommentSty{emph}
\SetProcNameSty{emph}
\DontPrintSemicolon

\IncMargin{5pt}

\journal{arXiv.org}

\begin{document}

\begin{frontmatter}

\title{On the Tractability of Minimal Model Computation \\ for Some CNF Theories}

\author[unical]{Fabrizio Angiulli}\ead{f.angiulli@dimes.unical.it}
\author[jce]{Rachel Ben-Eliyahu-Zohary}\ead{rbz@jce.ac.il}
\author[unical]{Fabio Fassetti}\ead{f.fassetti@dimes.unical.it}
\author[unical]{Luigi Palopoli}\ead{palopoli@dimes.unical.it}

\address[unical]{DIMES, University of Calabria, Rende (CS), Italy}
\address[jce]{Software Engineering Dept., Jerusalem College of Engineering, Jerusalem, Israel}
\begin{abstract}
Designing algorithms capable of {\em efficiently
constructing minimal models of CNFs} is an important task in AI.
This paper provides new results along this research line and presents new
algorithms for performing minimal model finding and checking over positive propositional
CNFs and model minimization over propositional CNFs. 
An algorithmic schema, called the {\em Generalized Elimination Algorithm} (GEA)
is presented, that computes a minimal model of {\em any} positive CNF.
The schema generalizes the {\em Elimination Algorithm}
(EA) \cite{Ben-Eliyahu-ZoharyP97}, which computes a minimal model of positive
\textit{head-cycle-free} (HCF) CNF theories. While the EA always runs in polynomial
time in the size of the input HCF CNF, the complexity of the GEA depends on the
complexity of the specific {\em eliminating operator} invoked therein, which may
in general turn out to be exponential. Therefore, a specific eliminating
operator is defined by which the GEA computes, in polynomial time, a minimal
model for a class of CNF that strictly includes head-elementary-set-free (HEF) CNF
theories \cite{GebserLL06}, which form, in their turn, a strict superset of HCF
theories.
Furthermore, in order to deal with the high complexity associated with
recognizing HEF theories, an ``incomplete'' variant of the GEA (called IGEA) is
proposed: the resulting schema, once instantiated with an appropriate
elimination operator, always constructs a model of the input CNF, which is
guaranteed to be minimal if the input theory is HEF.
In the light of the above results,
the main contribution of this work is the
enlargement of the tractability
frontier for the minimal model finding and checking
and the model minimization problems.
\end{abstract}


\end{frontmatter}

\section{Introduction}\label{sect:intro}

Minimal models play a vital role in many systems that are dedicated 
to knowledge representation  and reasoning.
The concept of minimal model is at the heart of several tasks
in Artificial Intelligence including
circumscription \cite{circ1,circ2,lif-circ}, default logic \cite{Rei80},
minimal diagnosis \cite{dKMR92}, planning \cite{KautzMS96},
and in answering queries posed on logic programs under the stable model semantics \cite{gl:negation,BiFr87}
and deductive databases under the generalized closed-world assumption
\cite{minker}.

On the more formal side, the task of reasoning with minimal models has
been the subject of several studies
\cite{Cad92a,Cad92b,KoPa90,EiGo93,CT93,BeDe96b,Be05,KiKo03}.
Given a propositional CNF
theory $\Pi$,
among others, the tasks of {\em Minimal Model Finding} and 
{\em Minimal Model Checking} have been considered. The former task consists
of computing a minimal model of $\Pi$,
the latter one is the problem of checking
whether a given set of propositional letter is indeed a minimal model for $\Pi$.

Findings regarding the complexity of reasoning with minimal
models show that these problems are intractable in the general case.
Indeed, it turns out that even when the theory is positive (that is, it does not contain constraints),
finding a minimal model is
$\rm P^{NP[O(\log n)]}$-hard \cite{Cad92a} (note that positive theories always have a minimal
model)\footnote{We recall that $\rm P^{NP[O(\log n)]}$ is the class of decision problems
that are solved by polynomial-time bounded deterministic Turing machines
making at most a logarithmic number
of calls to an oracle in $\rm NP$. For a precise characterization of the
complexity of model finding, given in terms of complexity classes of functions,
see \cite{CT93}.},
and checking whether a model is minimal for a given theory
is co-NP-complete \cite{Cad92b}.

The above formidable complexities characterizing the two
above mentioned problems
have motivated several researchers to look for heuristics
\cite{LoTr06,BeDe96b,Be05,AvBe07}
as long as, due to the complexity results listed above and to the still
unresolved
P vs NP conundrum, all exact algorithms for solving these problems remain
exponential in the worst case.

One orthogonal direction of research concerns singling out significant fragments
of CNF theories for which dealing with minimal models is tractable.
The latter approach has also the merit of providing
insights that can help improve  the efficiency of heuristics for
the general case.
For instance, algorithms designed
for a specific subset of general CNF theories can be
incorporated into algorithms for computing
minimal models of general CNF theories \cite{BeDe96b,SKC96,GSS03,HHST04}.

Within this scenario, in \cite{Ben-Eliyahu-ZoharyP97}
efficient algorithms are presented for computing and checking minimal models
of a restricted subset of positive CNF theories, called
{\em Head Cycle Free} (HCF) theories \cite{BenEliyahuD94}.
To illustrate, HCF theories are positive CNF theories satisfying the constraint that
there is no cyclic dependence involving two positive literals occurring in the
same clause.
Head-cycle-freeness can also be checked efficiently \cite{BenEliyahuD94}.
These results have been then exploited by other authors to
improve model finding algorithms for general theories.
For example, the system {\tt dlv}
looks for HCF fragments into general disjunctive logic programs to 
be processed in order to improve efficiency \cite{LRS97,KoLe99}.

The research presented here falls into the groove traced in \cite{Ben-Eliyahu-ZoharyP97}.
The central contribution of this work
is a
polynomial time algorithm for computing a minimal model for (a superset of) the class of positive
HEF (Head Elementary-Set Free) CNF theories, the definition
of which we adapt from the homonym one given in
\cite{GebserLL06} for disjunctive logic programs and
which form, in their turn, a strict superset of the class of HCF theories
studied in \cite{Ben-Eliyahu-ZoharyP97}.

To the best of our knowledge
positive HCF theories form the largets class of CNFs
for which a polynomial time algorithm 
solving the \textit{Minimal Model Finding} problem
is known so far.
Since HCF theories are a strict subset of HEF ones,
our main contribution is the
enlargement of the tractability
frontier for the minimal model finding problem.

It is worth noting that a relevant difference holds here that while HCF theories are recognizable in polynomial time,
for HEF ones the same task is co-NP-complete \cite{FassettiP10}.
Although
this undesirable property seems to reduce the 
applicability of the above result,
we will show that our approach leads to techniques to compute 
a model of any positive CNF theory in polynomial time, 
while the computed model is guaranteed to be minimal 
at least for all positive HEF theories.
Notice that this latter property holds
without the need to recognize whether the input theory is HEF or not.

The rest of the paper is organized as follows. 
In Section \ref{sect:probl_contrib}, we provide preliminary
definitions about CNF theories, present the problems
and the sub-classes of CNF theories of interest here,
depict contributions of the work, and discuss
application examples.
In Section \ref{sect:algo}, we introduce the Generalized Elimination Algorithm (GEA), 
that is the basic algorithm presented in this paper, 
and the concept of {\em eliminating operator} 
that it makes use of. 
Then, in Section \ref{sect:hef}, we formally define HEF CNF theories
and then construct an eliminating operator that enables GEA to compute a minimal model for 
a positive HEF CNF theory in polynomial time. 
In Section \ref{sect:beyond}, we study 
the behavior the GEA when applied to a general CNF theory
and introduce the Incomplete GEA which is able
to 
compute a minimal model for 
a positive HEF CNF theory in polynomial time
without the need to know in advance whether the input
theory is HEF or not.
Concluding remarks 
are provided in Section \ref{sect:conclusions}. 
For the sake of presentation, some of the intermediate result proofs 
are reported in the Appendix.

\section{Our problems and application scenarios}
\label{sect:probl_contrib}

In this section, first we define the problems we are dealing with in this paper
and then depict some application scenarios.

\begin{table}\small
\centering
 \begin{tabular}[h]{|c||p{0.7\textwidth}|}
 \hline
 \bf Symbol & \bf Description \\
 \hline
 $\P$ & A CNF theory \\
 \hline
 $\Pnd$ & A non-disjunctive theory obtained from $\P$ by deleting all disjunctive clauses\\
 \hline
 $\atom{\P}$ & The set of atoms appearing in $\P$ \\
 \hline
 $c_X$ & The clause obtained by projecting the clause $c$ on the set of atoms
$X$: if $c\equiv H\leftarrow B$ then $c_X \equiv H_X \leftarrow B_X$ with
$H_X = H\cap X$ and $B_X = B\cap X$ \\
 \hline
 $c_{X\leftarrow}$ & The clause obtained by projecting the head of the clause
$c$ on the set of atoms $X$: if $c\equiv H\leftarrow B$ then $c_{X\leftarrow}
\equiv H_X \leftarrow B$ with $H_X =
H\cap X$\\
 \hline
 $\P_X$ & The set of all the non-empty head clauses $c_X$ with $c$ in $\P$\\
 \hline
 $\P_{X\leftarrow}$ & The set of all the non-empty head clauses $c_{X\leftarrow}$
with $c$ in $X$\\
 \hline
 $\simpl{\P}{\M}$ & The set of all the non-empty clauses $c_\M$ with $c$ in
$\P$\\
 \hline
 $\ssimpl{\P}{\M}$ & A shortcut for $(\simpl{\P}{\M})_{\M\setminus\S}$\\
 \hline
  $\G(\P)$ & The dependency graph associated with the theory $\P$\\
 \hline
  $\eG(\P)$ & The elementary graph associated with the non-disjunctive theory
$\P$\\
 \hline
 \end{tabular}
\caption{Summary of the symbols employed throughout the paper.}
\label{table:symbols}
\end{table}

\subsection{Preliminary definitions}\label{sect:prelim}

In this section we recall or adapt the definitions of propositional CNF
theories and their subclasses (head-cycle-free, head-elementary-set-free) which
are of interest here.

An \textit{atom} is a propositional \emph{letter} (aka, \textit{positive} literal).
A \textit{clause}
(aka, \textit{rule} -- in the following we shall make use of the two terms interchangeably) is an expression of the form $H \leftarrow B$, where $H$ and $B$ are
sets of atoms\footnote{We prefer to adopt the implication-based syntax for clauses in the place of the more usual disjunction-based one to slightly ease the foregoing presentation.}.
$H$ and $B$ are referred to as,
respectively, the \textit{head} and \textit{body} of the clause;
the atoms in $H$ are also called head atoms
while the atoms in $B$ are also called body atoms.
With a little abuse of terminology, if $|H|>1$, we shall say the clause is
\textit{disjunctive}, otherwise it is a \textit{Horn}, or \textit{non-disjunctive}\footnote{We will use the terms \textit{Horn} and \textit{non-disjunctive} interchangeably.}.
Moreover, if $|H| = 1$ the clause is called \textit{single-head}.
A \textit{fact} is a single-head rule with empty body.
A \textit{theory} $\P$ is a finite set of clauses.
If there is some disjunctive rule in
$\P$ then $\P$ is called \textit{disjunctive}, otherwise it is called
\textit{non-disjunctive}.
$\atom{\P}$ denotes the set of all the atoms occurring in $\P$.
A set $S$ of atoms is called a \textit{disjunctive
set} for $\P$ if there exists at least one rule $H \leftarrow B$ in $\P$ such that $|H\cap S|>1$.
A \textit{constraint} is an empty-head clause.
A theory $\P$ is said to be \textit{positive}
if no costraint occurs in $\P$.

The semantics of CNF theories relies on the concepts of {\em interpretation} and {\em model}, which are recalled next.
An \textit{interpretation} $I$ for the theory $\P$ is a  set of atoms from $\P$. An atom is {\em true} (resp., {\em false})
in the interpretation $I$ if $a \in I$ (resp., $a \not\in I$). A rule $H \leftarrow B$ is true in $I$ if
either at least one atom occurring in $H$ is true in $I$ or at least one atom occurring in $B$ is false in $I$. An
interpretation $I$ is a \textit{model} for a theory $\P$ if all clauses occurring in $\P$
are true in $I$. A model $M$ for $\P$ is \textit{minimal} if no proper subset of $M$ is a
model for $\P$.

A directed graph $\G(\P)$, called \emph{positive dependency graph}, can be
associated with a
theory $\P$. Specifically, nodes in $\G(\P)$ are associated with atoms
occurring in $\P$ and, moreover, there is a directed edge $(m,n)$ from a node
$m$ to a node
$n$ in $\G(\P)$ if and only if there is a clause
$H \leftarrow B$ of $\P$ such that
the atom associated with $m$ is in $B$ and the atom associated with $n$
is in $H$.

Given a clause $c\equiv H\leftarrow B$ and a set of atoms $X$,
$c_{X\leftarrow}$ denotes the clause $H\cap X\leftarrow B$,
whereas $c_{X}$ denotes the clause $H\cap X\leftarrow B\cap X$.
Given a theory $\P$ and a set of atoms $X$, the theory $\P_{X\leftarrow}$
includes all \textit{non-empty head} clauses $c_{X\leftarrow}$, with $c$ a
clause in $\P$.
Analogously, the theory $\P_{X}$
includes all \textit{non-empty head} clauses $c_{X}$, with $c$ a clause in $\P$.
Given a theory $\P$, the theory $\P^{nd} \subseteq \P$ includes
all Horn clauses of $\P$.
In the following, we assume that the operators
$\cdot_X$ and $\cdot_{X\leftarrow}$ have precedence
over the operator $\cdot^{nd}$, thus that the expression
$\P^{nd}_X$ ($\P^{nd}_{X\leftarrow}$, resp.)
is to be intended equivalent to
$(\P_X)^{nd}$ ($(\P_{X\leftarrow})^{nd}$, resp.).

Table \ref{table:symbols} summarizes some of the 
symbols
used throughout the paper (some of them 
are defined in subsequent sections).

\begin{table}[t]
\small
\centering
\begin{tabular}{|l|c|c|c|c|c|}
 \hline
 \it Class of CNF Theory & REC & MFP & MMP & MMCP & MMFP \\ 
 \hline\hline
 \it General & --- & NP-{\it h} & $\rm P^{NP[O(\log n)]}$-{\it h} & coNP & $\rm \Sigma^P_2$-{\it h} \\
 \hline
 \it Positive general & P & FP & $\rm P^{NP[O(\log n)]}$-{\it h} & coNP & $\rm P^{NP[O(\log n)]}$-{\it h} \\
 \hline
 \it HEF & coNP & NP-{\it h} & \bf FP & \bf P & NP-{\it h} \\
 \hline
 \it Positive HEF & coNP & FP & \bf FP & \bf P & \bf FP \\
 \hline
 \it HCF & P & FP & FP & P & FP \\
 \hline
\end{tabular}
\caption{Problems and their complexity.}
\label{table:probl_compl}
\end{table}

\subsection{Problems}\label{sect:model_finding}

Table \ref{table:probl_compl} summarizes the problems and the classes of
CNF theories of interest here 
and reports the associated complexities.

As for the classes of CNF theories,
other than general one here we consider
HEF and HCF theories:
\begin{itemize}
\item[---]
\textit{Head Cycle Free} (HCF) theories \cite{BenEliyahuD94}
are CNF theories such that no connected component of
the associated dependency graph contains two positive literals occurring in the
same clause;

\item[---]
\textit{Head Elementary-Set Free} (HEF) CNF theories, the definition
of which we adapt from the homonym one given in
\cite{GebserLL06} for disjunctive logic programs (see Section \ref{sect:hef} for the formal definition of HEF theories), 
form a strict superset of the class of HCF theories.
\end{itemize}

The problems (listed in the table) are:
\begin{itemize}
\item[---] \textit{Recognition Problem} (REC):
Given a CNF theory $\P$ and a class $\cal C$ of CNF theories, 
decide if $\P$ belongs to the class $\cal C$;

\item[---] \textit{Model Finding Problem} (MFP):
Given a CNF theory $\P$, compute a model $\M$ for $\P$;

\item[---] \textit{Model Minimization Problem} (MMP):
Given a CNF theory $\P$ and a model $\M$ for $\P$,
compute a minimal model $\MM$ for $\P$ contained in $\M$;

\item[---] {\em Minimal Model Checking Problem} (MMCP):
Given a CNF theory $\P$ and a model $\M$ for $\P$,
check if $\M$ is indeed a minimal model for $\P$;

\item[---] {\em Minimal Model Finding Problem} (MMFP):
Given a CNF theory $\P$, compute a minimal model $\M$ for $\P$.
\end{itemize}
The MFP problem is NP-hard unless the theory is positive.
Indeed, in the latter case, the set consisting of all the literals
occurring in the theory is always a model.

In this work we will focus on the MMP, MMCP, and MMFP problems.

As for MMFP,
it turns out that, over positive CNF theories, this
is hard to solve. In particular, it is known that on positive theories
MMFP is ${\rm P}^{\NP[O(\log n)]}$-hard \cite{Cad92a} (even though positive CNF
theories always have a minimal model!).

Given a CNF theory $\P$ and a model $\M$ for $\P$,
it is worth noticing that the theory $\P_\M$
is always a positive CNF
and {that the models of $\P_\M$ are a subset of those of $\P$}.
This explains the fact that the complexity of the MMP and MMCP
problems, which have in input a model $\M$ other than the theory $\P$,
does not depend on positiveness of the theory.

Moreover, we notice that MMFP is not easier than MMP and MMCP since the latter problems
can be reduced to the former one as follows:
\begin{itemize}
\item[---] As for MMP, return ${\rm MMFP}(\P_\M)$;

\item[---] 
As for MMCP,
return true if ${\rm MMFP}(\P_\M) = \M$ and false otherwise.
\end{itemize}
Thus, if for a certain class of theories the MMFP were 
tractable, then both MMP and MMCP would become tractable as well.

Moreover, if attention is restricted to positive theories,
the MMP and MMFP problems coincide (since this time
MMFP can be reduced to MMP by setting $\M$
to the set of all the literals occurring in $\P$) and, consequently,
MMP on general theories is equivalent to MMFP on positive theories.

We notice that, on the other hand, for head-cycle-free CNF theories
things are easier than for the general case: indeed it was
proved in \cite{Ben-Eliyahu-ZoharyP97} that the MMFP 
is solvable in polynomial time if the input theory is HCF.

All that given, the following section details the contributions of the paper.

\subsection{Contributions and algorithms road map}

In this work we investigate the MMP and MMCP problems
on CNF theories and the MMFP on positive CNF theories.

Among the main contributions offered here, we will show that
MMP and MMCP are tractable on generic HEF theories,
while MMFP is tractable on positive HEF theories.

In order to provide a uniform treatment of these problems,
we will concentrate on algorithms for the MMP,
which can be considered the most general of them since its input consists of
both a CNF theory and a (not necessarily) non-minimal model of the theory.
Specifically, we provide a polynomial time algorithm
solving the MMP on general HEF CNF theories which, because of the observations made above, can be directly used to solve in polynomial time
the following five problems (see also 
cells of Table \ref{table:probl_compl} reported in bold):
($i$) MMP on non-positive HEF CNF theories,
($ii$) MMCP on non-positive HEF CNF theories,
($iii$) MMP on positive HEF CNF theories,
($iv$) MMCP on positive HEF CNF theories, and
($v$) MMFP on positive HEF CNF theories.

Also already noticed, differently from HCF theories, which turn out to
be recognizable in polynomial time \cite{BenEliyahuD94},
recognizing HEF theories is an intractable problem \cite{FassettiP10}.
This undesirable property may seem to limit the applicability of the above complexity results.
However, as better explained next,
we show that our MMP algorithm can be fed with any CNF theory $\Pi$ and any model $\M$ of $\Pi$ and it is
guaranteed to corectly minimize $\M$ at least in the case that the theory $\Pi$ is HEF.
Notice that this property holds without the need to recognize whether the input theory is HEF or not.

To illustrate, we start by presenting an algorithmic schema, called the
{\em Generalized Elimination Algorithm (GEA)} for model minimization
over CNF theories.
The GEA invokes a suitable \textit{eliminating operator}
in order to converge towards a minimal model of the input theory.
Intuitively, an eliminating operator is any function that,
given a model as the input, returns a model strictly included therein,
if one exists.
Therefore, the actual complexity of the GEA depends on
the complexity of the specific {eliminating operator} one decides to employ.
Clearly, the trivial eliminating operator may enumerate
(in exponential time)
all the interpretations contained in the given model
and check for satisfiability of the theory,
while we shall consider actually interesting only those 
eliminating operators that accomplish
their task in polynomial time.

A specific eliminating operator, denoted by $\xi_{\rm HEF}$, is henceforth defined,
by which the GEA computes a minimal model of any HEF theory in polynomial time.
However, the intractability of the recognition
problem for HEF CNF theories
may seem to narrow the applicability
of the results sketched above and to
reduce their significance to a mere theoretical result.
This seemingly relevant limitation can fortunately be overcome by
suitably readapting the structure of our algorithm:
to this end, we introduce the
\textit{Incomplete Generalized Elimination Algorithm} (IGEA)
that, once instantiated with a suitable operator, outputs a model of the
input theory, which is guaranteed to be minimal at least over HEF theories.

The design of IGEA leverages on the notion of \textit{fallible eliminating operator},
which is defined later in this paper.
Then, by coupling IGEA with the $\xi_{\rm HEF}$ operator,
we call this instance of the algorithm $\rm IGEA_{\xi_{\rm HEF}}$,
we obtain a polynomial-time algorithm
that always minimizes the input model of a HEF CNF
theory without the need of knowing in advance
whether the input CNF theory is HEF or not.
As for non-HEF theories, we show
that $\rm IGEA_{\xi_{\rm HEF}}$
always returns a model of the input theory which may be minimal or not, 
depending on the structure of the input theory.
This kind of behavior on non-HEF theories
is clearly the expected one since, as already noticed, recognizing HEF theories
is co-NP-complete.
Interestingly, this latter characteristics of $\rm IGEA_{\xi_{\rm HEF}}$
further enhances
its relevance, since its application is not restricted
to the class of HEF CNF theories, but to a even broader class thereof.

\begin{figure}[t]
\centering 
\begin{tabular}{cc}
\fbox{
$\begin{array}{rrll}
\P = \{  & g \vee j    & \leftarrow & \\
        & f \vee h    & \leftarrow & \\
        & b    & \leftarrow a & \\
        & c    & \leftarrow b & \\
        & a    & \leftarrow c & \\
        & d    & \leftarrow a, b & \\
        & c    & \leftarrow d & \\
        & e    & \leftarrow b & \\
        & h    & \leftarrow b & \\
        & f    & \leftarrow e, i & \\
        
        & i    & \leftarrow e, j & \\
        
        & g    & \leftarrow f & \\
        & e    & \leftarrow g & \\
        & j    & \leftarrow e & \\
        & h    & \leftarrow j & \\
        & j    & \leftarrow h & \\
        & c    & \leftarrow h, e & \} \\
\end{array}$} & 
 \qquad
\fbox{
\begin{minipage}{0.45\textwidth}
\includegraphics[width=1.0\textwidth]{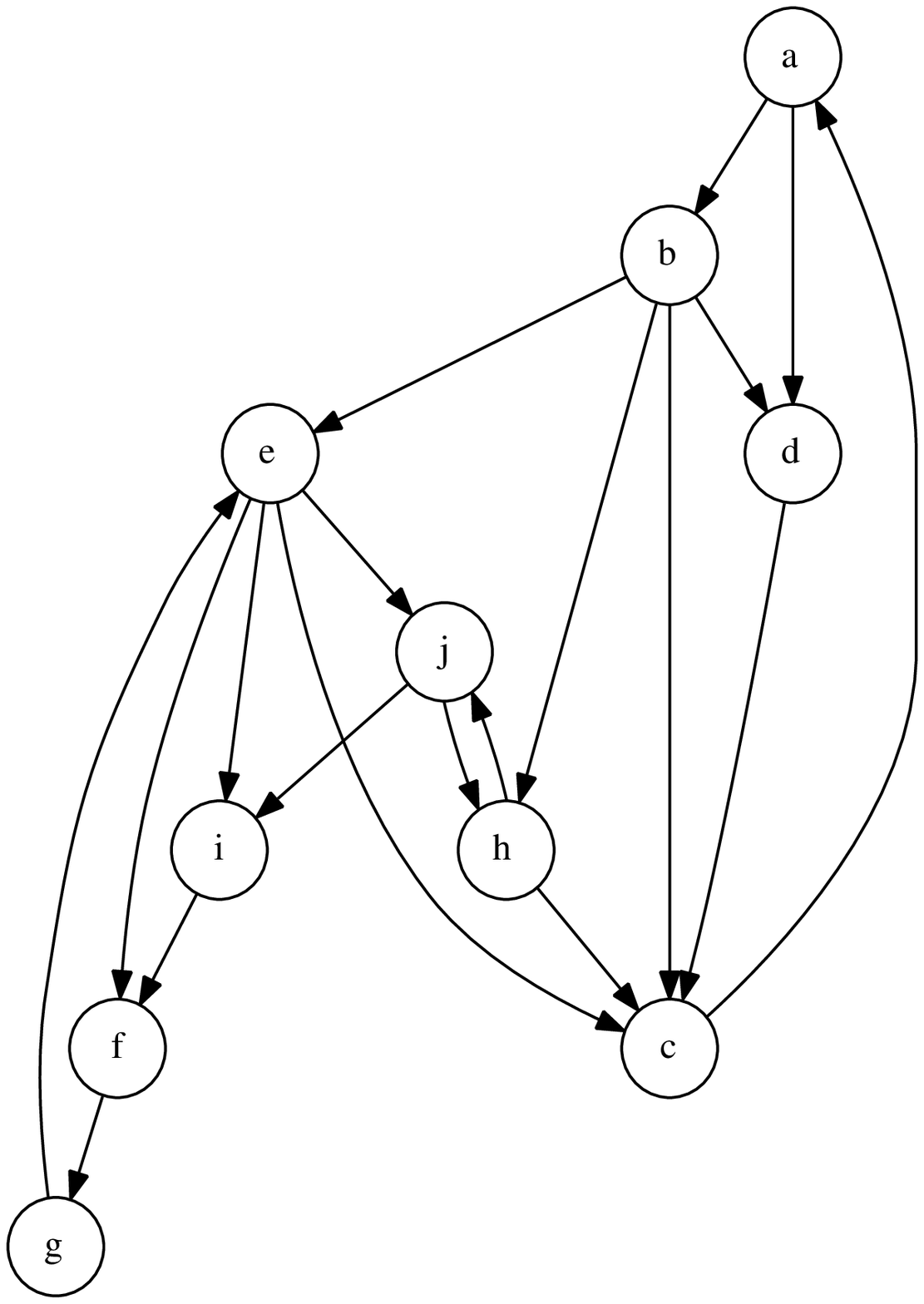}
\end{minipage}
}
\end{tabular}
\caption{A positive CNF theory and the associated dependency graph.}
\label{fig:ex_poscnf}
\end{figure}

\subsection{Application scenarios}
\label{sect:appl_scen}

In this section
we consider generic CNF theories
without concentrating on the particular class
(that is, general, HEF or HCF)
they belong to.
Later, in Section \ref{sect:appl_scen_hef}, 
we specialize some of the examples provided next
in the context of HEF theories, which is a main focus in our investigation.

As already noticed, the minimal model finding problem
is a formidable one 
and remains intractable even in the case attention
is restricted to positive CNFs. 
The following positive CNF theory 
will be employed in order
to describe the various concepts introduced
throughout the paper.

\begin{myexample}{Minimal models of positive CNF theories}\rm
\label{ex:pos_cnf}
Figure \ref{fig:ex_poscnf} reports an example of positive CNF theory $\P$ (on the left)
together with the associated dependency graph $\G(\P)$ (on the right).
The set $\atom{\P}$ is $\{a$, $b$, $c$, $d$, $e$, $f$, $g$, $h$, $i$, $j\}$ and it is the largest model
of $\P$. This theory has several models, but only a minimal one, which is
$\{j,h\}$.
\end{myexample}

\begin{figure}[t]
\centering 
\[\begin{array}{cc}
\fbox{$\begin{array}{rrll}
P = \{  & d    & \leftarrow \nbd c & \\
        & b    & \leftarrow a, e, \nbd d & \\
        & a    & \leftarrow b, e, \nbd d & \\
        & a    & \leftarrow c   & \\
        & b    & \leftarrow c  & \\
        & c    & \leftarrow a, b & \\
        & a, b & \leftarrow \nbd f & \}
\end{array}$} & 
\qquad
\fbox{$\begin{array}{rrll}
\M = \{ & a, d &  \} \\ \\
P^\M = \{ & d    & \leftarrow & \\
         & a    & \leftarrow c   & \\
         & b    & \leftarrow c  & \\
         & c    & \leftarrow a, b & \\
         & a, b & \leftarrow & \}
\end{array}$}
\end{array}\]
\caption{A logic program $P$, a model $\M$ of $P$, and the reduct $P^\M$.}
\label{fig:ex_stable}
\end{figure}

To illustrate a setting in which 
positive CNFs natural arise, consider Logic Programming,
a central tool in Knowledge Representation and Reasoning.
In the field of
Logic Programming, the notion of negation by 
default poses the problem of defining a proper notion
of model of the program. Among the several proposed semantics for logic programs
with negation, the
Stable Models and Answer Sets semantics are nowadays the reference one for
closed world scenarios \cite{GelfondL88}.
An interesting
application of our techniques
concerns stable model (or answer set) checking. 
To illustrate, stable models 
exploit the concept of the reduct of the program, as clarified in the following definition.

\begin{definition}[Stable Model \cite{GelfondL88}]\rm 
Given a logic program $P$ and a model $\M$ of $P$, 
the \textit{reduct} of $P$ w.r.t $\M$, also denoted by $P^\M$, is the program built from
$P$ by ($i$) removing all rules that contain a negative literal $not~a$ in the
body with $a \in \M$, and ($ii$) removing all negative literals from the remaining
rules.
A model $\M$ of $P$ is {\em stable} if $\M$ is a minimal model of $P^\M$.
\end{definition}

\begin{myexample}{Stable Models of Logic Programs}
\label{ex:stable_models}
Figure \ref{fig:ex_stable} shows, on the left, a logic program $P$
and, on the right, the reduct $P^\M$ of $P$ w.r.t. the model $\M=\{a,d\}$.
In this case, 
$\M$ is a minimal model of $P^\M$
and, hence, it is a stable model of $P$.
\end{myexample}

It is worth noticing that $P^\M$ is a CNF since,
by definition of the reduct, negation by default does not occur in any clause
of $P^\M$. Moreover, $\M$ is always a model of $P^\M$ and is given in input.

Therefore, by setting $\Pi=P^\M$ the problem 
of verifying if a 
given model $\M$ for the logic program $P$ is stable
fits the minimal model checking problem
for positive CNFs and, as such, can be suitably dealt with using the techniques this paper proposes.

\begin{figure}[t]
\centering 
\[\begin{array}{cc}
\fbox{$\begin{array}{rrll}
\P = \{  & b    & \leftarrow a & \\
        & c    & \leftarrow a & \\
        & a    & \leftarrow b, c & \\
        & b, c & \leftarrow & \\
        & d    & \leftarrow\\
        &      & \leftarrow b, d & \} \\ ~\\ ~\\ ~\\ ~\\ 
\end{array}$} & 
\qquad
\fbox{$\begin{array}{rrll}
\P^+ = \{  & b    & \leftarrow a & \\
        & c    & \leftarrow a & \\
        & a    & \leftarrow b, c & \\
        & b, c & \leftarrow & \\
        & d    & \leftarrow\\
        & \phi & \leftarrow b, d & \\
        & a & \leftarrow \phi & \\
        & b & \leftarrow \phi & \\
        & c & \leftarrow \phi & \\
        & d & \leftarrow \phi & \}
\end{array}$}
\end{array}\]
\caption{A CNF $\P$ and its positive form $\P^+$.}
\label{fig:ex_posform1}
\end{figure}

In order to analyze a different application scenario, let us assume a positive CNF is given.
Next we show that the given theory can be indeed reduced to a positive theory
whose models have some clear relationship with the models of the original theory.

Let us first consider the definition of positive form of a CNF.

\begin{definition}[Positive Form of a CNF theory]\label{def:positive_form}\rm
The theory $\P^+$, also said the \textit{positive form} of $\P$, is defined as follows:
(1) for each clause $H \leftarrow B$ of $\P$, 
if $H$ is not empty then the clause $H \leftarrow B$ is in $\P^+$;
(2) for each clause $\leftarrow B$ of $\P$, the clause $\phi \leftarrow B$ is in $\P^+$;
(3) for each atom $a$ occurring in $\P$, the clause $a\leftarrow\phi$
is in $\P^+$.
\end{definition}

The following result relates models of $\P$ with
minimal models of $\P^+$.

\begin{proposition}
Given a CNF theory $\P$, 
{if $\phi$ belongs to the (unique) minimal model of $\P^+$
then $\P$ is inconsistent, otherwise the set of
minimal models of $\P$ and $\P^+$ coincide}. 
\end{proposition}

\begin{proof}

First of all, we will observe that each model of $\P$ is a model of $\P^+$ as
well and, then, $\phi$ is not in $\M$.

\begin{observation}\label{obs:model}
Let $\M$ be a model for $\P$ and consider the theory $\P^+$.
All the clauses (1) in $\P^+$ are also in $\P$ and then are true.
Since $\M$ is a model for $\P$ all the empty-head clauses of $\P$ don't have
the body fully contained in $\M$ and, therefore, $\M$ satisfies all the clauses (2) of
$\P^+$. Finally, since $\phi$ is not in $\M$ all the clauses (3) are true.
\end{observation}

Now, let $\M^+$ be a minimal model of $\P^+$ and $\atom{\P^+}$ be the set of all
atoms occurring in $\P^+$.

Note that, because of the presence of the set of clauses (3), two cases are possible, that are:
either $\M^+$ contains $\phi$ and then all the atoms occurring in $\P^+$; or
$\M^+\subset \atom{\P^+}$ and, in particular, $\phi \not \in \M^+$.

\begin{enumerate}
 \item As for the first case, if $\P$ had a model $\M$ then, due to
Observation \ref{obs:model}, $\M$ would be a model of $\P^+$ as well and then
$\M^+$ would not be minimal. Thus, $\P$ is inconsistent.
 \item As for the second case, $\M^+$ does not contain $\phi$. Consider now the
theory $\P$. All the non-empty-head clauses in $\P$ are also in $\P^+$ and,
then, are satisfied by $\M^+$. Consider, now, the empty-head clauses in $\P$.
Because of the presence of clauses (2) in $\P^+$, and since $\M^+$ does not contain
$\phi$, it is the case that the body of such clauses is not fully contained in
$\M^+$. Thus, the correspondent clauses in $\P$ are satisfied by $\M^+$.
\end{enumerate}

This implies that $\M^+$ is a model of $\P$ as well.
\end{proof}

To illustrate, consider the following example.

\begin{myexample}{General CNF theories}
\label{ex:posform}
Consider the CNF reported in Figure \ref{fig:ex_posform1}
on the left. In the same figure, one the right, it is reported
the positive form $\P^+$ of $\P$.
$\P$ has only one minimal model, namely $\{c, d\}$,
which is precisely the unique minimal model of $\P^+$.
\end{myexample}

\section{Generalized Elimination Algorithm}\label{sect:algo}

In this section, a generalization of the elimination algorithm proposed in \cite{Ben-Eliyahu-ZoharyP97},
called Generalized Elimination Algorithm, is introduced.
We begin by providing some preliminary concepts, notably, those of \textit{steady set} and \textit{eliminating operator}.

Intuitively, given a model $\M$ for a theory $\P$, the steady set is the subset
of $\M$ containing atoms which ``cannot'' be erased from $\M$, for otherwise
$\M$ would no longer be a model for $\P$. As proved next, the steady set can be
obtained by computing the model of a certain non-disjunctive theory.

\begin{definition}[Steady set]
Given a CNF theory $\P$ and a model $\M$ for $\P$,
the minimal model $\S \subseteq \M$ of the theory $\P^{nd}_{M\leftarrow}$ 
is called the \emph{$\steady$ set} of $\M$ for $\P$.
\end{definition}

Note that the steady set $\S$ of $\M$ for $\P$ always
exists and is unique.
Indeed, $\P^{nd}_{M\leftarrow}$ is a Horn positive CNF
and it is known that these kinds of theories
have one and only one minimal model (which can be computed in polynomial time)
\cite{DowlingG84,Lloyd1987}.

\begin{property}[$\mathcal{MM}$-containment]\label{prop:mm_cont}
Given a positive CNF theory $\P$, a model $\M$ for $\P$
and the steady set $\S$ of $\M$ for $\P$,
it holds that each model of $\P$ contained in $\M$ contains $\S$.
\end{property}
\begin{proof}
First, notice that the models of the positive CNF theory $\P$
which are contained in the model $\M$ of $\P$
coincide with the models of the positive CNF theory $\P_{\M\leftarrow}$.

Since $\Pnd_{\M\leftarrow}$ is contained in $\P_{\M\leftarrow}$,
by monotonicity of propositional logic, it follows that all logical
consequences of $\Pnd_{\M\leftarrow}$
are also logical consequences of $\P_{\M\leftarrow}$ and, hence,
each model of $\P_{\M\leftarrow}$ contains the
unique minimal model of $\Pnd_{\M\leftarrow}$, which is the steady set of $\M$
for $\P$.
\end{proof}

\begin{definition}[Erasable set]
Let $\M$ be a model of a positive CNF theory $\P$.
A non-empty subset $\E$ of $\M$ is said to be
\emph{erasable} in $\M$ for $\P$
if $\M \setminus \E$ is a model of $\P$.
\end{definition}

The following result holds.

\begin{proposition}
Let $\M$ be a model of a positive CNF theory $\P$,
let $\S$ be the steady set of $\M$ for $\P$,
and let $\E$ be a set  erasable in $\M$ for $\P$.
Then, $\E\subseteq\M\setminus\S$.
\end{proposition}
\begin{proof}
For the sake of contradiction, assume that $\E\cap\S\neq\emptyset$. Then,
$\M\setminus\E$ is a model of $\P$ that does not contain $\S$,
which contradicts the fact that $\S$ has the $\mathcal{MM}$-containment property
in $\M$ for $\P$
(See Property \ref{prop:mm_cont}).
\end{proof}

\begin{definition}[Eliminating operator]
Let $\M$ be a model of a positive CNF  theory $\P$.
An \emph{eliminating operator} $\xi$ is a mapping that,
given $\M$ and $\P$ in input, returns an erasable set in $\M$ for $\P$,
if one exists, and an the empty set, otherwise.
\end{definition}

It immediately follows that
if $\xi(\P,\M) = \emptyset$ then $\M$ is a minimal
model of $\P$.
This is easily shown by observing that
$\xi(\P,\M) = \emptyset$ implies that there is no erasable set in $\M$,
namely no subset of $\M$ is a model for $\P$.

We are now ready to present our algorithmic schema,
referred to as the Generalized Elimination Algorithm (GEA)
throughout the paper,
which is summarized in Figure \ref{fig:elimination_algo}.
Note that GEA has an operator $\xi$ as its parameter\footnote{The term {\em schema}
is used here since actual algorithms are obtained only
after instantiating the generic $\xi$ operator invoked in
the GEA to a specific operator.}.

\newcommand{\MMin}{{\M^\ast}}

\begin{figure}[t]
\begin{algorithm}[H]
\footnotesize
\KwIn{A CNF theory $\P$ and a model $\M$ of $\P$}
\KwOut{A minimal model $\MMin$ of $\P$ contained in $\M$}
\BlankLine
remove all constraints from $\P$\;
$stop$ = $false$\;
\Repeat{$stop$ \label{line:end_out_cycle}}{\label{line:start_out_cycle}
	compute the minimal model $\S$ of $\Pnd_{\M\leftarrow}$\;
	\If{$\S$ is a model of $\P$}{
		$\MMin = \S$
		$stop$ = $true$\;
	}\Else{
		$\E = \xi(\P,\M)$\;
		\If{$(\E = \emptyset)$}{
			$\MMin = \M$\;
			$stop$ = $true$\;
		}\Else{
			$\M = \M \setminus \E$\label{line:compute_xi_S}\;
		}
	}
}
\Return $\MMin$
\caption{Generalized Elimination Algorithm with operator $\xi$, $GEA_{\xi}(\P,\M)$}
\label{algo:elimination}
\end{algorithm}
\caption{Generalized Elimination Algorithm with operator $\xi$, $\rm GEA_{\xi}(\P,\M)$}
\label{fig:elimination_algo}
\end{figure}

Our first result states that GEA is correct under the condition that the operator parameter $\xi$ is an eliminating operator.

\begin{theorem}[GEA correctness]\label{theo:gea_correct}
Let $\P$ be a CNF theory and $\M$ be a model of $\P$.
If $\xi$ is an eliminating operator,
then the set returned by $\rm GEA_{\xi}$ on input $\P$ and $\M$ 
is a minimal model for $\P$ contained in $\M$.
\end{theorem}

\begin{proof}
First of all, since $\M$ is a model of $\P$, by definition of model 
all the constraints (aka empty-head clauses)
of $\P$ are true in $\M$ and
are also true in any subset of $\M$. Hence, they can be disregarded
during the subsequent steps (see line 1).

Moreover, note that,
by definition of steady set, it follows that the set $\S$
computed at the beginning of each iteration of the algorithm (line 4)
is a (not necessarily proper) subset of every minimal model contained in $\M$.
Let $n$ be the number of atoms in the model $\M$ computed at line 1
of the GEA.

Three cases are possible, which are discussed next:
\begin{enumerate}
\item {\em $\S$ is a model of $\P$.} Since $\S$ is the steady set of $\M$ for $\Pi$,
if $\S$ is a model for $\P$, then it is also minimal; so the algorithm stops
and returns a correct solution.\footnote{We recall that if the steady set $\S$
of M for $\P$ is 
a model of $\P$ then it is
the unique minimal model of $\P$ contained in $\M$.
Hence, the test at line $4$ serves the purpose of accelerating
the termination of the algorithm. However, operations in lines
3-6 could be safely dropped without affecting the
correctness of the algorithm.}

\item {\em ${\E} = \emptyset$}. By definition of eliminating operator, if $\E$
is empty, then $\M$ is a minimal model;
so the algorithm stops and returns a correct solution.

\item {\em $\E \neq \emptyset$.} In this case,
a non-empty set of atoms is deleted from $\M$,
letting (by definitions of eliminating operator and erasable set)
$\M$ still be a model for $\P$.
Thus, at the next iteration, the algorithms
will work with a smaller (possibly not minimal) model $\M$.
Hence, after at most $n$
iterations, either case 1 or case 2 applies.
\end{enumerate}
\end{proof}

The next result states the time complexity of the GEA that, clearly, will depend
on the complexity $C_\xi$ associated with the evaluation of the eliminating
operator $\xi$.

\begin{proposition}\label{prop:gea_cost}
Let $n$ and $m$ denote the number of atoms occurring in the heads of $\P$ and,
overall, in $\P$, respectively.
Then, for any model $\M$ of $\P$,
$\rm GEA_{\xi}(\P,\M)$ runs in time $O(n m + n C_\xi)$.
\end{proposition}
\begin{proof}
Since at each iteration (if the stopping condition is not matched) at least one
atom is removed,
the total number of iterations is $O(n)$.
As for the cost spent at each iteration, the dominant operations are:
$(i)$ computing the (unique) minimal model of a non-disjunctive theory (line 4)
which can be accomplished in linear time w.r.t. $m$ by the well-known unit
propagation procedure \cite{DowlingG84};
$(ii)$ checking if a set of atoms is a model (line 5)
which can be accomplished in linear time in $m$ as well;
$(iii)$ applying the eliminating operator (line 8), whose cost is $C_\xi$. This
closes the proof.
\end{proof}

In particular,
consider the naive operator $\xi_{\exp}$
that enumerates all the $2^n$ non-empty subsets of $\M$
and either returns one of these, call it $\E$,
such that $\M \setminus\E$ is a model for $\P$,
or an empty set if such a set $\E$ does not exist. The resulting algorithm
$\rm GEA_{\xi_{\exp}}$ returns a minimal model
of $\P$ but requires exponential running time.

\medskip
Conversely,
as an example of instance of the GEA algorithm
having polynomial time complexity on a specific class 
of CNF theories,
consider
the Elimination
Algorithm presented in \cite{Ben-Eliyahu-ZoharyP97}.
This algorithm can be obtained from the GEA 
by having the operator $\xi_{\rm HCF}$ (described next) as 
the eliminating 
operator $\xi$
and the set $\atom{\P}$ as the input model $\M$.
Indeed,
as shown in \cite{Ben-Eliyahu-ZoharyP97},
the Elimination Algorithm
computes a minimal model of a positive HCF theory in
polynomial time. The definition of
$\xi_{\rm HCF}$ operator follows \cite{Ben-Eliyahu-ZoharyP97}.
Let $\P$ be a 
{positive} HCF CNF
theory and let $\M'$ be the set of the heads of the disjunctive rules in
$\P_{\M\leftarrow}$ which are false in $\M$. Then,
$\xi_{\rm HCF}(\P, \M)$ is defined to return a \textit{source} of $\M'$,
where a source of the set of atoms $\M'$ is a connected component in the subgraph of
$\G(\P)$ induced by $\M'$ which does not have incoming arcs.

\bigskip
Before leaving this section, we provide two further results 
which will be useful when discussing the MMCP and the MMFP.

\begin{lemma}\label{lemma:LemmaCheck}
Given a CNF theory $\P$, an eliminating operator $\xi$ and a model $\M$
of $\P$, $\M$ is minimal for $\P$ if and only if $\rm GEA_{\xi}(\P,\M)$ outputs $\M$.
\end{lemma}
\begin{proof}
The proof follows by noticing that GEA always
outputs a (possibly non-proper) minimal sub-model of the initial model $\M$ as
its output. 
\end{proof}

\begin{lemma}\label{lemma:LemmaFind}
Given a positive CNF theory $\P$ and an eliminating operator $\xi$, 
then $\rm GEA_{\xi}(\P,\atom{\P})$ outputs a minimal model of $\P$.
\end{lemma}
\begin{proof}
The proof follows by noticing that 
$\atom{\M}$ is a model of $\M$, being $\P$ a positive theory,
and by Lemma \ref{lemma:LemmaCheck}.
\end{proof}

\section{Model minimization on HEF CNF theories}\label{sect:hef}

We have noticed above that the complexity of GEA
depends on the complexity characterizing, in its turn,
the specific elimination operator it invokes.
On the other hand, the MMP
being ${\rm P}^{\NP[O(\log n)]}$-hard \cite{Cad92a} implies
that, unless the polynomial hierarchy collapses, the GEA will
generally require exponential time to terminate when called
on a generic input CNF theory.
Therefore, it is sensible to single out significant subclasses
of CNF theories for which it is possible to devise a specific
eliminating operator guaranteeing a polynomial running time
for the GEA.

In this respect, it is a simple consequence of the results presented in
\cite{Ben-Eliyahu-ZoharyP97} that a model
of any head-cycle-free theory can be indeed minimized in polynomial time using the
Elimination Algorithm.
So, the interesting question remains open of whether we can
do better than this. Our answer to this question is affirmative
and this section serves the purpose of illustrating this result.
In particular, we shall show that by carefully defining the
eliminating operator, we can have that the GEA
minimizes 
in polynomial time
a model of any HEF CNF theory.
In Section \ref{sect:beyond}, we shall moreover show that there also exist
CNF theories which are not HEF but for which the algorithm,
equipped with a proper eliminating operator, efficiently minimizes a 
model.

\subsection{Head-elementary-set-free theories and $\sel$ sets}
\label{HEF}

Next, we recall the definition of head-cycle-free theories \cite{BenEliyahuD94},
adapt that of head-elementary-cycle-free theories \cite{GebserLL06} to our
propositional context 
and provide a couple of preliminary results which will be useful in the
following.

We proceed by introducing the concepts of outbound and elementary set.

\begin{definition}[Outbound Set (adapted from \cite{GebserLL06})]
Let $\P$ be a CNF theory. For any set $Y$ of atoms occurring in $\P$, a subset $Z$ of $Y$ is \emph{outbound} in $Y$ for $\P$
if there is a clause $H \leftarrow B$ in $\P$ such that: {\rm({\it i})} $H\cap Z\neq \emptyset$; {\rm({\it ii})} $B\cap (Y\backslash Z)\neq \emptyset$; {\rm({\it iii})} $B\cap Z = \emptyset$ and {\rm({\it iv})} $H\cap (Y\backslash Z) = \emptyset$.
\end{definition}
Intuitively, $Z \subseteq Y$ is outbound in $Y$ for $\P$ if there exists a rule
$c$ in $\P$ such that the partition of $Y$ induced by $Z$ (namely, $\langle Z;
Y\setminus Z \rangle$) ``separates'' head and body atoms of $c$.

\begin{myexample}{Outbound set}\label{ex:elementary_set}
Consider the theory
\begin{eqnarray*}
\begin{array}{rrll}
\P = \{& b, c & \leftarrow a   & \\
            & b    & \leftarrow c   & \\
            & c    & \leftarrow b   & \\
            & a    & \leftarrow b   & \\
            & d    & \leftarrow b,c & \}
\end{array}
\end{eqnarray*}
and the set $E=\{a,b,c\}$. Consider, now, the subset $O=\{a,b\}$ of $E$. $O$ is outbound in $E$ for $\P$ because of the clause $b\leftarrow c$, since $c \in E \setminus O$, $c\not\in O$, $b\in O$ and $b\not\in E \setminus O$.
\end{myexample}

Let $O$ be a non-outbound set in $X$ for $\P$.
$O$ is \textit{minimal non-outbound} if any proper subset $O' \subset O$
is outbound in $X$ for $\P$.

\begin{definition}[Elementary Set (adapted from \cite{GebserLL06})]
Let $\P$ be a CNF theory. For any non-empty set $Y \subseteq \atom{\P}$, $Y$ is
\emph{elementary} for $\P$ if all non-empty proper subsets of $Y$
are outbound in $Y$ for $\P$.
\end{definition}

For example, the set $E_{ex}$ of Example \ref{ex:elementary_set} is elementary
for the theory $\P_{ex}$, since each non-empty proper subset of $E_{ex}$ is
outbound in $E_{ex}$ for $\P_{ex}$.

\begin{definition}[Head-Elementary-Set-Free CNF theory (adapted from \cite{GebserLL07})]\label{def:hef_theory}
Let $\P$ be a CNF theory. $\P$ is \emph{head-elementary-set-free
(HEF)} if for each clause $H \leftarrow B$ in $\P$, there is no elementary set
$E$ for $\P$ such that $|E\cap H|>1$.
\end{definition}
So, a CNF theory $\P$ is HEF if there is no elementary set containing two or more
atoms appearing in the same head of a rule of $\P$.
An immediate consequence of Definition \ref{def:hef_theory} is the following
property.

\begin{property}\label{prop:hef_def}
A theory $\P$ is not HEF if and only if
there exists a set $X$ of atoms of $\P$ such that $X$ is both a disjunctive
and an elementary set for $\P$.
\end{property}

For instance, the theory $\P_{ex}$ of Example \ref{ex:elementary_set} is not
HEF, since for the rule $b,c\leftarrow a$ and the elementary set
$E_{ex}$, we have $|E_{ex} \cap \{b,c\}| > 1$.

\subsection{Examples of HEF theories}
\label{sect:appl_scen_hef}

Now the examples already introduced in Section \ref{sect:appl_scen}
are discussed in the context of HEF CNF theories.

\begin{exampleContinued}{\ref{ex:pos_cnf}}{Minimal models of positive CNF theories}\rm
Consider the theory reported in Figure \ref{fig:ex_poscnf}.
This is an HEF CNF theory since no superset of $\{g,j\}$ and 
no superset of $\{f,h\}$ is an elementary set for $\P$.
\end{exampleContinued}

\begin{exampleContinued}{\ref{ex:stable_models}}{Stable Models of Logic Programs}
A logic program $P$ is HEF if the CNF $\widehat{P}$ obtained by removing
all the literals of the form $not~a$ from the body of its rules
is HEF \cite{GebserLL07}.

Importantly, it holds that if the logic program $P$
is HEF and $\M$ is a model of $P$, then also $P^\M$ is HEF.
This follows since, by definition, a logic program $P$ is HEF if and only if
the CNF $\widehat{P}$ is HEF, and by Lemma \ref{theo:HEF_monotonicity}
(reported in Section \ref{sect:sel_set}) any subset of clauses of 
a HEF CNF is HEF as well, and $P^\M$ is precisely a subset of $\widehat{P}$.

Notably, even if $P$ is not HEF, it could be anyway the case that $P^\M$ is HEF,
and this broadens the range of applicability of the techniques proposed here.

As an example, consider again Figure \ref{fig:ex_stable}.
The program $P$ there reported is not HEF,
since the set $S=\{a,b,c\}$ is both disjunctive and elementary.

Conversely, $P^\M$ is HEF
since the set $S$ is no longer elementary 
because the subsets $\{a,c\}$ and $\{b,c\}$ of $S$ are not outbound in $S$.

{Moreover, we notice that the subgraph of $\G(P^\M)$
induced by $S$ is a connected component 
and then both $P$ and $P^\M$ are not HCF.}
\end{exampleContinued}

\begin{exampleContinued}{\ref{ex:posform}}{General CNF theories}
Given a non-positive CNF $\P$, it holds that if $\P$ is not HEF,
then also $\P^+$ is not HEF.

Let $\P'$ be the subset of $\P$ obtained by
removing the contraints in $\P$.
Notice that $\P'$ can be obtained from $\P^+$ by
first removing the clauses of the form $a\leftarrow\phi$
for each $a\in\atom{\P}$
(see point (3) of Definition \ref{def:positive_form})
and then projecting it on $\atom{\P}$.
Since the HEF property does not depend on 
the costraints,
it follows from Lemma \ref{theo:HEF_monotonicity}
(reported in Section \ref{sect:sel_set}) 
that if $\P$ is not HEF, then $\P^+$ is not HEF as well.

Conversely, if $\P$ is HEF, then $\P^+$ can happen to be either HEF or not.

As an example, consider the theories displayed in Figure \ref{fig:ex_posform1} of Section \ref{sect:appl_scen}. In this case
$\P$ and $\P^+$ are both HEF.
Conversely, consider the theories reported in Figure \ref{fig:ex_posform2}.
In this case, $\P$ is HEF, whereas $\P^+$ is not.
\end{exampleContinued}

\begin{figure}[t]
\centering 
\[\begin{array}{cc}
\fbox{$\begin{array}{rrll}
\P = \{  & b    & \leftarrow a & \\
        & c    & \leftarrow a & \\
        & a    & \leftarrow b, c & \\
        & b    & \leftarrow c & \\
        & b, c & \leftarrow & \\
        & d    & \leftarrow\\
        &      & \leftarrow b, d & \\ 
        &      & \leftarrow c, d & \} \\ 
        ~\\ ~\\ ~\\ ~\\ 
\end{array}$} & 
\qquad
\fbox{$\begin{array}{rrll}
\P^+ = \{  & b    & \leftarrow a & \\
        & c    & \leftarrow a & \\
        & a    & \leftarrow b, c & \\
        & b    & \leftarrow c & \\
        & b, c & \leftarrow & \\
        & d    & \leftarrow\\
        & \phi & \leftarrow b, d & \\
        & \phi & \leftarrow c, d & \\
        & a & \leftarrow \phi & \\
        & b & \leftarrow \phi & \\
        & c & \leftarrow \phi & \\
        & d & \leftarrow \phi & \}
\end{array}$}
\end{array}\]
\caption{An HEF CNF $\P$ and its positive form $\P^+$
which is not HEF.}
\label{fig:ex_posform2}
\end{figure}

\subsection{$\Sel$ sets}

We introduce next the definition of
\emph{simplified theory}
and of $\emph{\sel}$ set
that will play a relevant role in
the definition of the eliminating 
operator for HEF theories.

\begin{definition}[Simplified theory]
Let $\P$ be a CNF theory and $\M$ be a model of
$\P$. Then 
the \emph{simplified theory of $\P$ w.r.t. $\M$},
denoted as $\ssimpl{\P}{\M}$,
is the CNF theory
$\left(\simpl{\P}{\M}\right)_{\M\setminus\S}$,
where 
$$\simpl{\P}{\M} = \{ H\leftarrow B \in \P : H\cap\S=\emptyset \mbox{ and } \M\supseteq B \}$$
and
$\S$ is the steady set of $\M$ in $\P$.
\end{definition}

The clauses in $\simpl{\P}{\M}$ are those clauses of $\P$ having the body
fully contained in $\M$ and some atoms of the head contained in $\M$ but not in
$\S$.
Note that it cannot be the case for the head of any clause in $\P$
to have empty intersection with $\M$
(or, analogously, the head is empty) since, in such a case, $\M$ would not be a
model for $\P$.
Then, intuitively, $\simpl{\P}{\M}$ contains the subset of the clauses of
$\P$ which could be falsified if atoms would be eliminated from the model $\M$,
so that we would have a model for $\P$ no longer.
Note that, we do not consider the case that atoms of $\S$ are eliminated from
$\M$ since, by definition of steady set, if any atom of $\S$ were eliminated
we would have no longer models for $\P$ in $\M$.
Simplified theories enjoy two useful properties. 

As for the first, we observe that, for any CNF theory $\P$ and model $\M$ of
$\P$, $\simpl{\P}{\M}$ is positive.

The second one, summarized in the following Lemma, tells that
$\simpl{\P}{\M}$ contains no facts.

\begin{lemma}\label{prop:equiv_th_atom}
Let $\P$ be a CNF theory, let $\M$ be a model of $\P$, and
let $\S$ be the steady set of $\M$ for $\P$.
Then no clause of the form $h\leftarrow$, with
$h$ a single letter, occurs in the theory
$\ssimpl{\P}{\M}$.
\end{lemma}

Next, we introduce the notion of $\sel$ set which will be used for defining
the eliminating operator for HEF theories.

\begin{definition}[$\Sel$ set]
Given a CNF theory $\P$ and a set $X \subseteq \atom{\P}$,
$X$ is {\em $\sel$} for $\P$ if $X$ is
both an elementary set for $\P$ and
a non-outbound set in $\atom{\P}$ for $\P$.
\end{definition}

Intuitively, a {\em $\sel$} set $X$ for $\P$ is a set of atoms such that for no
disjunctive clause $c$ in $\P$, the body of 
$c$ is satisfied by atoms not occurring in $X$ and its head is contained in $X$
(as will be clear in the proof of Theorem \ref{theo:eliminable_set}). 
Notice that, as a consequence, no clause may become unsatisfied
by removing a {\em $\sel$} set $X$ from a model.

\subsection{On the erasability properties of $\Sel$ sets}\label{sect:sel_set}

Next, we are going to show that, given any theory $\P$ and model $\M$ of $\P$,
any {\em $\sel$} set is erasable in $\M$ for $\P$. In order to do that, we
shall:
\begin{enumerate}
\item 
demonstrate a one-to-one correspondence
between the erasable sets in $\M$ for $\P$ and the erasable
sets in $\M\setminus\S$ for $\ssimpl{\P}{\M}$,
where $\S$ is the steady set of $\M$ for $\P$ (Lemma \ref{th:equiv_theories}),
\item
show that the property of a theory $\P$ being HEF is retained by the subsets of
$\P$ (Lemma \ref{theo:HEF_monotonicity}): 
this implies that if a theory $\P$ is HEF then, for each model $\M$ of $\P$, also $\ssimpl{\P}{\M}$ is HEF,
\item prove that, for any HEF theory $\P$ and any model $\M$ of $\P$, any $\sel$
set is erasable in $\ssimpl{\P}{\M}$ (Theorem \ref{theo:eliminable_set}),
whereby the sought result is obtained.
\end{enumerate}

Le following results are conducive the the achievement of the aforementioned 
objectives. To ease readability, some of the proof are reported in the appendix.

\begin{lemma}\label{th:equiv_theories}
Let $\P$ be a CNF theory, let $\M$ be a model of $\P$
and let $\S$ be the steady set of $\M$ for $\P$.
A set of atoms $\E$ is erasable in $\M$ for $\P$ if and only if
$\E$ is erasable in $(\M\setminus\S)\supseteq\atom{\ssimpl{\P}{\M}}$ for $\ssimpl{\P}{\M}$.
\end{lemma}

\begin{lemma}\label{theo:HEF_monotonicity}
Let $\P$ be a HEF CNF theory.
For each set of clauses $\P'\subseteq\P$ and
for each set of atoms $X$, the theory $\P'_X$ is HEF.
\end{lemma}

\begin{theorem}\label{theo:eliminable_set}
Let $\P$ be a HEF CNF theory,
let $\M$ be a model for $\P$ and
let $\S$ be the steady set of $\M$ for $\P$,
If $\E \subseteq \atom{\ssimpl{\P}{\M}}$
is $\sel$ for $\ssimpl{\P}{\M}$
then $\E$ is erasable in 
$\M$ for $\P$.
\end{theorem}
\begin{proof}
By Lemma \ref{th:equiv_theories} it suffices to prove 
that $\E$ is erasable in
$\atom{\ssimpl{\P}{\M}}$ for $\ssimpl{\P}{\M}$,
which is accounted for next.

First of all, recall that $\ssimpl{\P}{\M}$
is a positive theory.
Moreover, by Lemma \ref{theo:HEF_monotonicity}, $\ssimpl{\P}{\M}$ is HEF, since $\P$ is
HEF.

Clearly, $\atom{\ssimpl{\P}{\M}}$ is a model of $\ssimpl{\P}{\M}$.
It must be proved that
each clause of $\ssimpl{\P}{\M}$
is true in $\atom{\ssimpl{\P}{\M}}\setminus\E$.
Let $H \leftarrow B$ be a generic clause of $\ssimpl{\P}{\M}$ such that
$\E$ contains $H$: this is the only kind of clause that might become false in
$\atom{\ssimpl{\P}{\M}}\setminus\E$.
Next, it is proved that $H\leftarrow B$ is true in 
$\atom{\ssimpl{\P}{\M}}\setminus\E$.
First notice that, by definition of $\ssimpl{\P}{\M}$, it cannot be the case that
$H$ is empty and $|B|\ge 1$.
Thus, the following three cases have to be considered:
\begin{enumerate}
\item
{\em $B$ is empty and $|H|=1$.}
By Lemma \ref{prop:equiv_th_atom},
such a clause cannot exist.

\item
{\em $B$ is empty and $|H|>1$.}
Notice that,
since $\E$ is an elementary set for $\ssimpl{\P}{\M}$,
it cannot be the case
that $|H\cap\E|>1$ or, in other words, that $\E\supseteq H$,
since the theory
$\ssimpl{\P}{\M}$ is HEF.
Hence, the clause $H\leftarrow B$ is true also in 
$\atom{\ssimpl{\P}{\M}}\setminus\E$;

\item
{\em $B$ is not empty.}
By contradiction, assume that $H\leftarrow B$ is false in
$\atom{\ssimpl{\P}{\M}}\setminus\E$.
Then, $H\leftarrow B$ is such that $H \subseteq \E$
and $B \subseteq \atom{\ssimpl{\P}{\M}} \setminus \E$,
namely, none of the atoms in $B$ occurs in $\E$.
But this rule cannot exist,
since $\E$ is non-outbound in $\atom{\ssimpl{\P}{\M}}$
for $\ssimpl{\P}{\M}$.
\end{enumerate}
\end{proof}

\subsection{On the existence of a $\sel$ set in a HEF theory}
\label{sect:sel_exists}

Next, we are going to show that, under the condition that $\atom{\P}=\atom{\Pnd}$, any HEF theory $\P$ has a  
$\sel$ set. This result, stated as Theorem \ref{theo:existence} below, shall be attained by preliminarly proving that:
\begin{enumerate}
\item 
a $\sel$ set for the non-disjunctive subset of a positive CNF theory is also $\sel$ for the whole theory (Lemma \ref{theo:pnd->p}),
\item
each  HEF CNF theory $\P$ has a \sel~set (Lemma \ref{theo:minimal->elem}).
\end{enumerate}

\begin{lemma}\label{theo:pnd->p}
Let $\P$ be a HEF CNF theory.
If $O\subseteq\atom{\Pnd}$ is $\sel$ for $\Pnd$ then $O$ is $\sel$ for $\P$.
\end{lemma}

\begin{lemma}\label{theo:minimal->elem}
Let $\P$ be a non-disjunctive CNF theory.
Each minimal non-outbound set in $\atom{\P}$ for $\P$
is $\sel$ for $\P$.
\end{lemma}

The following result eventually states another key property of $\sel$
sets in HEF CNF theories.

\begin{theorem}\label{theo:existence}
Let $\P$ be a disjunctive HEF CNF theory
such that $\atom{\P}=\atom{\Pnd}$.
Then, there exists a non-empty set of atoms $O\subseteq\atom{\P}$
such that $O$ is $\sel$ for $\P$.
\end{theorem}
\begin{proof}
Since $\P$ is a disjunctive HEF CNF theory,
it cannot be the case that $\atom{\P}=\atom{\Pnd}$
is elementary for $\Pnd$, for otherwise $\atom{\P}$ would be
elementary also for $\P$, implying that $\P$ is not HEF.
Since $\atom{\Pnd}$ is not elementary in $\Pnd$, by definition,
there exists a set of atoms which is non-outbound in $\atom{\Pnd}$
for $\Pnd$ and, in particular, there exists a minimal non-outbound set
$O\subset\atom{\Pnd}$ for $\atom{\Pnd}$ in $\Pnd$.
To conclude, $O$ is $\sel$ for $\Pnd$ by Lemma \ref{theo:minimal->elem}
and,
since $O\subset\atom{\Pnd}=\atom{\P}$,
$O$ is $\sel$ for $\P$
by Lemma \ref{theo:pnd->p}.
\end{proof}

\subsection{Computing a $\sel$ set}\label{sect:comp_sel_set}

This section is devoted to proving that a $\sel$ set of an HEF CNF theory can
be, in fact, computed in polynomial time.

\begin{figure}[t]
\begin{function}[H]
\footnotesize
\KwIn{A non-disjunctive CNF theory $\P$ \newline a set of atoms
$X\subseteq\atom{\P}$}
\KwOut{The elementary subgraph $\eG(\P,X)$ of $\P$}
\BlankLine
$i = 0$\;
$E_i = \emptyset$\;
$\eG_i = \langle X, E_i\rangle$\;
\Repeat{$C=\emptyset$}{
		let $C$ be the set of clauses $h \leftarrow B$ in $\P$ s.t. 
		the subgraph of $\eG_i$ induced by $B$ is strongly connected\;
		$E_{i+1} = E_i\cup\{(b,h) \mid b\in B \mbox{ and } h\leftarrow B\in
C\}$\;
		$\eG_{i+1} = \langle X, E_{i+1}\rangle$\;
		remove $C$ from $\P$\;
		$i = i + 1$\;
	}
\Return $\eG_i$
\label{fig:procPE}
\caption{compute\_elementary\_subgraph()}
\label{proc:compute_elem_set}
\end{function}
\caption{The \ref{proc:compute_elem_set} function}
\label{fig:compute_elem_set}
\end{figure}

The task of computing a $\sel$ set is accomplished by the function
\ref{funct:find_elem_nonout_set}
shown in
Figure \ref{fig:funct_find_sel}.

At each iteration, the function \ref{funct:find_elem_nonout_set}
makes use of the function \textit{\ref{proc:compute_elem_set}},
which is detailed in Figure \ref{fig:compute_elem_set}.
The latter function receives as input a theory $\P$ and a set of atoms $X$,
an returns a graph, also denoted by $\eG(\P,X)$, called
the \textit{elementary subgraph of $X$ for $\P$} \cite{GebserLL06}.
The function reported in Figure \ref{fig:compute_elem_set} is
substantially the same as that described at pag. 4 of \cite{GebserLL06}.
Specifically, in the pseudo-code, by $\langle X, E\rangle$ 
it is denoted a graph where $X$ is the set of nodes and
$E$ is the set of arcs.

\begin{myexample}{Elementary subgraph}
Figure \ref{fig:elem_graph_example} reports an example of computation of an
elementary subgraph.

Since $E_0 = \emptyset$, $\eG_0$ is a graph including nodes but no arcs
(see Figure \ref{fig:elem_graph_example}(c)).
The clauses in $\P$ whose body is fully contained in one strongly connected
component of $\eG_0$ are all the clauses with just one atom in the body,
namely $C=\{ c_1$, $c_2$, $c_3$, $c_5 \}$.
Thus, $E_1$ consists in set of arcs $\{(a, b), (c, a), (a, c), (d,a)\}$ and
the clauses $c_1$, $c_2$, $c_3$ and $c_5$ are removed from $\P$.

The graph $\eG_1$ is shown in Figure \ref{fig:elem_graph_example}(d).
The unique clause left in $\P$ whose body is fully contained in a strongly
connected is $c_4$, then $C=\{c_4\}$, 
$E_2 = \{(a, d), (c, d)\}$, and $c_4$ is removed from $\P$.

Figure \ref{fig:elem_graph_example}(e) reports the graph $\eG_2$.
Since the body of $c_6$ does not belong to a strongly connected component of
$\eG_2$, the procedure stops returning $\eG_2$ as the elementary subgraph 
$\eG(\P,X)$ of $X$ for $\P$.

\begin{figure}[h]
\centering
\subfloat[The program $\P$ and the set of atoms $X$]{
\begin{minipage}{0.5\textwidth}
\small
\begin{equation*}
\begin{array}{l@{\quad}l}
\begin{split}
\P=\{ c_1 \equiv & ~b \leftarrow a\\
      c_2 \equiv & ~a \leftarrow c\\
      c_3 \equiv & ~c \leftarrow a\\
      c_4 \equiv & ~d \leftarrow a, c\\
      c_5 \equiv & ~a \leftarrow d\\
      c_6 \equiv & ~e \leftarrow a, b\}\\
\end{split}
&
\begin{split}
 X = \{a, b, c, d, e\}
\end{split}
\end{array}
\end{equation*}
\end{minipage}
}\qquad
\subfloat[The dependency graph of $\P$]{
\begin{minipage}{0.33\textwidth}
\centering
\includegraphics[width=0.75\textwidth]{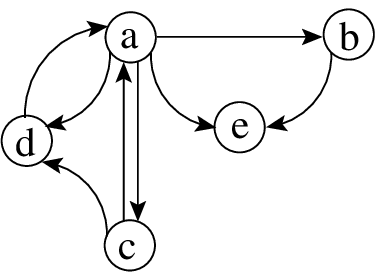}
\end{minipage}
}\\
\subfloat[$\eG_{0}=\langle X, E_0\rangle$]{
\begin{minipage}{0.33\textwidth}
\centering
\includegraphics[width=0.75\textwidth]{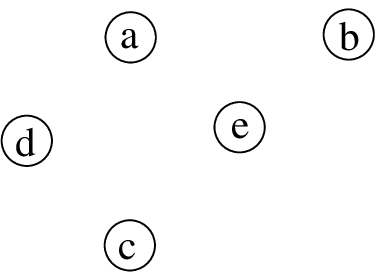}
\end{minipage}
}
\subfloat[$\eG_{1}=\langle X, E_1\rangle$]{
\begin{minipage}{0.33\textwidth}
\centering
\includegraphics[width=0.75\textwidth]{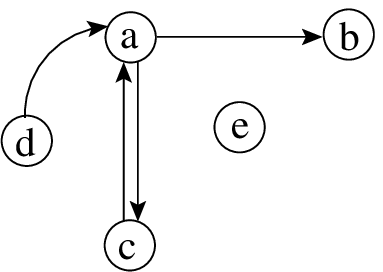}
\end{minipage}
}
\subfloat[$\eG_{2}=\langle X, E_2\rangle$]{
\begin{minipage}{0.33\textwidth}
\centering
\includegraphics[width=0.75\textwidth]{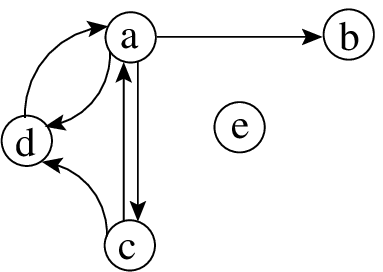}
\end{minipage}
}
 \caption{An example of elementary subgraph construction}
 \label{fig:elem_graph_example}
\end{figure}

\end{myexample}

Next, we recall the main result stated in \cite{GebserLL06},
concerning elementary subgraphs.

\begin{proposition}[Theorem 2 of \cite{GebserLL06}]\label{prop:elem_subgraph_connected}
For any non-disjunctive theory $\P$ and any
set $X$ of atoms occurring in $\P$, $X$ is an elementary set for
$\P$ if and only if the elementary subgraph of $X$ for $\P$
is strongly connected. 
\end{proposition}
Moreover, as also proved in \cite{GebserLL06}, the following proposition holds.
\begin{proposition}[\cite{GebserLL06}]\label{prop:elem_subgraph_cost}
The procedure \ref{proc:compute_elem_set} terminates in polynomial
time. 
\end{proposition}
Indeed, at each iteration, a non-empty set of clauses 
(for otherwise the algorithm would stop) is taken into account
and each clause of the theory is considered at most once. 
Thus, the number of iterations is
at most linear w.r.t. the number of clauses of the theory.
As for the cost of a single iteration,
we have first to find a clause $c$ such that the subgraph of $\eG_i$
induced by the body of $c$ is strongly connected. This task can be
clearly accomplished in polynomial time.
Second, we have to build the new graph $\eG_{i+1}$ by adding new arcs
to $\eG_i$, a task
that can be accomplished also in polynomial time.

\medskip
Let us now resort to the function \ref{funct:find_elem_nonout_set}
(see Figure \ref{fig:funct_find_sel}).

Assume that the set $\atom{\Pnd}$ is not elementary for $\Pnd$.
Then the elementary graph $\eG\left(\atom{\Pnd}, \Pnd\right)$ is not strongly
connected (by Proposition \ref{prop:elem_subgraph_connected}).
Therefore, the graph $\eG\left(\atom{\Pnd}, \Pnd\right)$ can be partitioned into
the sets $\C_1,\dots,\C_k$
of its maximal strongly connected components
and organized into $m\ge 1$ levels, such that
if there is an arc from a node in a connected component $C_i$
to a node in a connected component $C_j$, then
the level of $C_i$ precedes the level of $C_j$.
Isolated connected components possibly occurring in the graph are assumed to be
part of the last level $m$.

\begin{figure}
\begin{function}[H]
\KwIn{An 
HEF CNF theory $\P$ such that $\atom{\P}=\atom{\Pnd}$}
\KwOut{A $\sel$ set in $\atom{\P}$ for $\P$}
\BlankLine
$X_0 = \atom{\P}$\;
$i = 0$\;
$stop = false$\;
\Repeat{$stop$}{
	compute the elementary subgraph $\G_i = \eG(X_{i}, \Pnd)$\;
	\If{$\G_i$ is strongly connected}{
		$stop = true$\;
	}\Else{
	select a connected component $\C$ in the last level of $\G_i$\;
	$X_{i+1} = X_{i} \setminus{\C}$\;
	$i = i+1$\;
	}
}
\Return $X_i$
\caption{\textit{find\_\sel\_set}()}
\label{funct:find_elem_nonout_set}
\end{function}
\caption{The \ref{funct:find_elem_nonout_set} function.}
\label{fig:funct_find_sel}
\end{figure}

The following Theorem states the correctness of the function
\ref{funct:find_elem_nonout_set}.

\begin{theorem}\label{theo:correct_xiHEF}
Let $\P$ be a disjunctive HEF CNF theory
such that $\atom{\P}=\atom{\Pnd}$.
Then, 
the function \ref{funct:find_elem_nonout_set}{\rm (}$\P${\rm )}
computes a $\sel$ set for $\P$.
\end{theorem}

In order to prove the theorem the following result is useful.

\begin{claim}\label{claim:non_outbound}
For each $i\ge 0$, $X_i$ is a non-empty non-outbound set in $\atom{\Pnd}$ for $\Pnd$.
\end{claim}
\begin{proofOf}{Claim \ref{claim:non_outbound}}
The proof is by induction.

We start by noticing that the non-empty set $X_0=\atom{\P}=\atom{\Pnd}$ is non-outbound in $\atom{\Pnd}$ for $\Pnd$,
by definition of outbound set. 
Moreover, consider the graph $\G_0$, namely, the
elementary graph associated with the set of atoms $X_0 = \atom{\Pnd}$ and the
theory $\Pnd$.
Note that this graph is not strongly connected since $\P$ is, by hypothesis, a disjunctive HEF theory such that $\atom{\P}=\atom{\Pnd}$ and then
$\atom{\P}$ is not elementary for $\Pnd$.

Now, for $i>1$,
assume by induction hypothesis that $X_{i}$ is non-outbound in $\atom{\Pnd}$ for $\Pnd$
and that the graph $\G_i$ is not strongly connected
(for otherwise the algorithm would have stopped).
Consider a strongly connected component $\C$ of the last level of $\G_i$
and the set $X_{i+1} = X_{i}\setminus\C$.
Note that $X_{i+1}$ is non-empty since, by induction hypothesis, $\G_i$
is not strongly connected and note, moreover, that also $\C$ is not empty.

Next, it is shown that $X_{i+1}$ is non-outbound in $\atom{\Pnd}$ for $\Pnd$ or,
in other words, that there does not exist any clause $c\equiv h\leftarrow B$
such that
$B \subseteq X_0 \setminus X_{i+1}$ and $h \in X_{i+1}$,
(note that this means that, without loss of generality, we can limit ourselves to focus
only on such single-head clauses
where the atom in the head belongs to $X_{i+1}$ and the body is in $X_0 \setminus X_{i+1}$).

So, assume by contradiction that one such a clause $c$ indeed exists. Two cases
are possible.

\begin{description}
\item
$B \cap \C = \emptyset$. In this case, $B \subseteq X_0\setminus X_{i}$.
Therefore, $c$ cannot exist in $\Pnd$ since $X_{i}$
is non-outbound in $\atom{\Pnd}$ for $\Pnd$.
\item
$B\cap \C\neq\emptyset$.
Also in this case,  the clause $c$ cannot exist in $\Pnd$.
Indeed, the clause $c_{X_{i}}$
obtained by projecting $c$ on $X_{i}$, has its body contained in $\C$ and its
head in $X_{i+1}$.
Since this clause would belong to $\Pnd_{X_{i}}$, then it would be the case that
$\C$ would not belong to the last level of $\G_i$.

\end{description}

This concludes the proof of Claim \ref{claim:non_outbound}.
\end{proofOf}

Using Claim \ref{claim:non_outbound}, the statement of Theorem
\ref{theo:correct_xiHEF} easily follows, as shown next.

\begin{proofOf}{Theorem \ref{theo:correct_xiHEF}}
When the algorithm \ref{funct:find_elem_nonout_set}
stops, the last set $X_i$ is elementary for $\Pnd$,
since the graph $\G_i$ is strongly connected.
By Claim \ref{claim:non_outbound}, the set $X_i$
is also non-empty and non-outbound in $\atom{\P}$ for $\Pnd$.
To conclude, by Lemma \ref{theo:pnd->p}, the set $X_i$
is $\sel$ for $\P$.
\end{proofOf}

\begin{figure}[t]
\centering 
\begin{tabular}{cc}
\fbox{
\begin{minipage}{0.4\textwidth}
\includegraphics[width=1.0\textwidth]{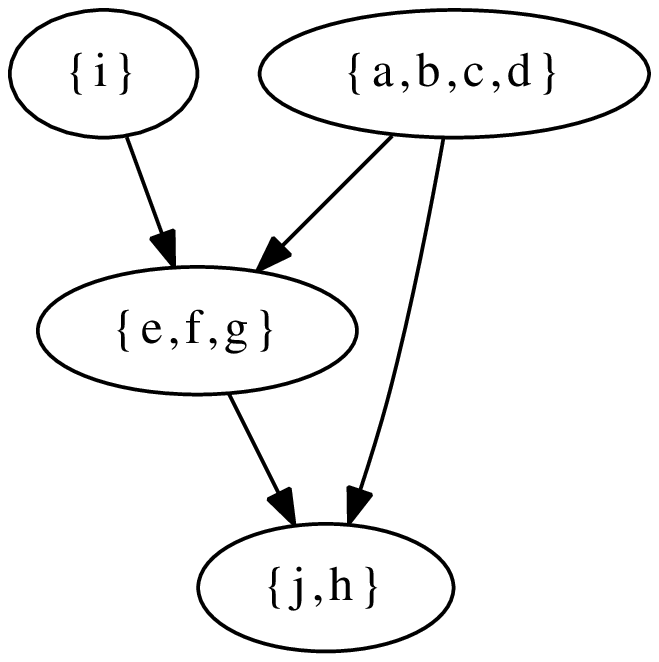}
\end{minipage}
}
& 
 \qquad
\fbox{
\begin{minipage}{0.4\textwidth}
\centering
\includegraphics[width=0.8\textwidth]{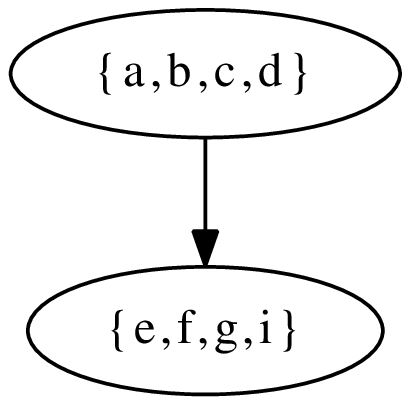}
\end{minipage}
}
\end{tabular}
\caption{Example of execution of the function \ref{funct:find_elem_nonout_set}.}
\label{fig:ex_poscnf_supelem}
\end{figure}

\begin{exampleContinued}{\ref{ex:pos_cnf}}{Minimal models of positive CNF theories}\rm
Consider again the theory $\P$ reported in Figure \ref{fig:ex_poscnf}
and the function \ref{funct:find_elem_nonout_set}($\P$).
The connected components of the elementary subgraph
$\G_0$ are shown in Figure \ref{fig:ex_poscnf_supelem} on the left.
Thus, there is a unique connected component in the last level
of $\G_0$, which is $\C=\{j,h\}$, and $X_1$ is set 
to $\{a,b,c,d,e,f,g,i\}$.
Notice that the 
connected components of the elementary subgraph
$\G_1$, which are reported in Figure \ref{fig:ex_poscnf_supelem} on the right,
are not a subset of those of $\G_0$.
The set $X_2$ is then $\{a,b,c,d\}$ and it is 
the $\sel$ set returned by the function.
\end{exampleContinued}

The next theorem accounts for the complexity of the function \ref{funct:find_elem_nonout_set}.

\begin{theorem}\label{theo:ptime_xiHEF}
For any CNF theory $\P$,
the function \ref{funct:find_elem_nonout_set}{\rm (}$\P${\rm )}
terminates in polynomial time in the size of the theory.
\end{theorem}
\begin{proof}
Initially $X_0$ contains all the atoms occurring in the input theory. Then, at
each iteration, either the graph $\G_i$ is strongly connected and then the
function stops and returns $X_i$, or $\G_i$ is not strongly
connected and in such case some node is removed from $X_i$.
In the latter case, there exist at least two strongly connected components in
graph $\G_i$.
$\C$ is one of them and is such that $X_i \supset\C\supset\emptyset$.
Thus, $X_{i+1}$ is always non-empty.
As for the convergence, it is ensured by the fact that the singleton set is
strongly connected by definition.

The number of iterations executed by the \ref{funct:find_elem_nonout_set}
function is at most equal to the number of atoms occurring in
the input theory, since in the worst case $\C$ consists in just one single atom
at each iteration.
The statement follows by the fact that each iteration can be accomplished in
polynomial time.
\end{proof}

\subsection{Defining an eliminating operator for HEF CNF theories}

In previous sections, we showed that:
\begin{itemize}
\item[--]
given a HEF CNF theory $\P$ and a model $\M$ for $\P$,
a $\sel$ set for $\ssimpl{\P}{\M}$ is erasable
in $\M$ for $\P$ (Theorem \ref{theo:eliminable_set} in Section \ref{sect:sel_set}),

\item[--]
given a HEF CNF theory $\P$,
if the set of atoms of $\P$
coincides with that of its non-disjunctive
fragment, a $\sel$ set always exists (see Theorem \ref{theo:existence} in Section \ref{sect:sel_exists})
and can be indeed computed in polynomial time (see Theorems \ref{theo:correct_xiHEF}
and \ref{theo:ptime_xiHEF} in Section \ref{sect:comp_sel_set}).
\end{itemize}

Putting things together, given an HEF CNF,
it can be concluded that 
if $\atom{\ssimpl{\P}{\M}}$ 
coincides with $\atom{\ssimplnd{\P}{\M}}$,
an erasable set $\E$ in $\M$ for $\P$
can be obtained by computing a $\sel$ set
for $\ssimpl{\P}{\M}$ (as detailed in Section \ref{sect:comp_sel_set}).

In order to build a suitable eliminating operator for HEF theories,
it remains to prove 
that
if $\atom{\ssimpl{\P}{\M}}$ is a strict superset of $\atom{\ssimplnd{\P}{\M}}$
then
it is always possible to find in polynomial time a model $\M'\subseteq\M$
such that $\atom{\ssimpl{\P}{\M'}}$ coincides with $\atom{\ssimplnd{\P}{\M'}}$.

\begin{proposition}\label{th:cnf_transf}
Given a CNF theory $\P$ and a model $\M$ for $\P$,
a model $\M'\subseteq\M$
such that 
$\atom{\ssimpl{\P}{\M'}}$ coincides with $\atom{\ssimplnd{\P}{\M'}}$
can be computed in polynomial time.
\end{proposition}

The above result, which is valid  not only for HEF CNF theories
but, rather, for any CNF
theory, will make the strategy above depicted generally applicable to any HEF
CNF theory.

In order to prove Proposition \ref{th:cnf_transf},
the intermediate results stated in technical 
Lemmas \ref{th:worthless_atoms} and \ref{theo:singleton_erasable} are preliminarily needed.

\begin{lemma}\label{th:worthless_atoms}
Let $\P$ be a CNF theory,
let $\M$ be a model of $\P$
and let $\S$ be the steady set of $\M$ for $\P$.
Then, $\E = (\M\setminus\S) \setminus
\atom{\ssimpl{\P}{\M}}$
is erasable in $\M$ for $\P$.
\end{lemma}

\begin{lemma}\label{theo:singleton_erasable}
Let $\P$ be a CNF theory
and let $\M$ be a model of $\P$.
If there exists an atom $a$ such that $a\in\M\setminus\atom{\Pnd_{\M\leftarrow}}$ then
$\{a\}$ is erasable in $\M$ for $\P$.
\end{lemma}

We are now in the position of proving
Proposition \ref{th:cnf_transf}.

\begin{proofOf}{Proposition \ref{th:cnf_transf}}
Let $\S$ denote the steady set 
of $\M$ in $\P$.
The two following transformations (see points 1-2)
can be recursively applied, till the 
condition $\atom{\ssimpl{\P}{\M}}=\atom{\ssimplnd{\P}{\M}}$ 
is met:
\begin{enumerate}
 \item
\underline{\em If $\M\setminus\S$ is a strict superset of $\atom{\ssimpl{\P}{\M}}$ then}
by Lemma \ref{th:worthless_atoms} the atoms in the non-empty set
$\E = (\M\setminus\S)\setminus\atom{\ssimpl{\P}{\M}}$ are erasable in $\M$
for $\P$ and $\M'$ can be set to $\M\setminus\E$;

 \item
\underline{\em Else if $\M\setminus\S$ is a strict superset of $\atom{\ssimplnd{\P}{\M}}$ 
then} any atom $a \in (\M\setminus\S)\setminus\atom{\ssimplnd{\P}{\M}}$ is such that
$\{a\}$ is erasable in $\M\setminus\S$ for $\ssimpl{\P}{\M}$ 
(by Lemma \ref{theo:singleton_erasable}, 
since $\M\setminus\S$ is a model for $\ssimplnd{\P}{\M}$)
and also erasable in $\M$ for $\P$ (by Lemma \ref{th:equiv_theories});
hence, let $\E=\{a\}$
an arbitrarily chosen atom in $(\M\setminus\S)\setminus\atom{\ssimplnd{\P}{\M}}$,
then $\M'$ can be set to $\M\setminus\E$;

 \item
\underline{\em Else}
it is the case that $\atom{\ssimpl{\P}{\M}}=\atom{\ssimplnd{\P}{\M}}$.
\end{enumerate}
The whole process can be completed polynomial time.
\end{proofOf}

\begin{figure}
\begin{function}[H]
\footnotesize
\KwIn{An HEF CNF theory $\P$ and a model $\M$ of $\P$}
\KwOut{An erasable set $\E$ in $\M$ for $\P$}
\BlankLine
$\E'=\emptyset$\;
\Repeat{$\Delta\E=\emptyset$}{
$\M=\M\setminus\E'$\;
Compute the steady set $\S$ of $\M$ for $\P$\;
$\Delta\E=\emptyset$\;
\If{$(\M\setminus\S)\supset\atom{\ssimpl{\P}{\M}}$}{
	$\Delta\E = (\M\setminus\S)\setminus\atom{\ssimpl{\P}{\M}}$
}\ElseIf{$(\M\setminus\S)\supset\atom{\ssimplnd{\P}{\M}}$}{
	Select an atom $a$ in $(\M\setminus\S)\setminus\atom{\ssimplnd{\P}{\M}}$\;
	$\Delta\E=\{a\}$\;
}
	$\E'=\E'\cup\Delta\E$\;
}
\If{$\ssimpl{\P}{\M}$ is non-disjunctive}{
	$\E'' = \M\setminus\S$\;
}\Else{
	$\E'' = find\_super{\rm -}elementary\_set$($\ssimpl{\P}{\M}$)\;
	
}
$\E=\E'\cup\E''$\;
\Return $\E$\;
\caption{$\xi_{HEF}$()}
\label{proc:hef_elim}
\end{function}
\caption{The $\xi_{HEF}$ eliminating operator.}
\label{fig:hef_elim}
\end{figure}

Before describing the $\xi_{HEF}$ eliminating operator, the following technical result is needed.

\begin{lemma}\label{theo:erasable_nd}
Let $\P$ be a CNF theory, let $\M$ be a model of $\P$, and
let $\S$ be the steady set of $\M$ for $\P$.
If the theory $\ssimpl{\P}{\M}$ is non-disjunctive,
then $\emptyset$ is its minimal model.
\end{lemma}

Figure \ref{fig:hef_elim} shows a realization of the $\xi_{HEF}$ eliminating operator.
The following theorem asserts the most relevant result of this section,
that is, that a minimal model for an HEF CNF theory can be indeed computed in
polynomial time.

\begin{theorem} \label{theo:GEApolynomial}
Let $\P$ be a HEF CNF theory and $\M$ be a model of $\P$. 
Then,
$\rm GEA_{\xi_{\rm HEF}}(\P,\M)$ computes, in polynomial time,
a minimal model of $\P$ contained in $\M$.
\end{theorem}
\begin{proof}
Because of Theorem \ref{theo:gea_correct} and Proposition \ref{prop:gea_cost},
in order to prove the statement, it is sufficient to show that
$(i)$ $\xi_{\rm HEF}$ returns an erasable set, 
if such a set exists, and an empty one otherwise (namely
that $\xi_{\rm HEF}$ is, in fact, an eliminating operator) and that
$(ii)$ $\xi_{\rm HEF}$ runs in polynomial time.

Let us consider first point ($i$).
Lines 2-12 in Figure \ref{fig:hef_elim} serve the purpose
of finding a subset
$\M'\subseteq\M$
such that 
$\atom{\ssimpl{\P}{\M'}}$ coincides with $\atom{\ssimplnd{\P}{\M'}}$
according to the strategy depicted in the proof of Proposition \ref{th:cnf_transf}
Notice that, the set $\E'=\M\setminus\M'$ is an erasable set.

We can now assume that 
$\atom{\ssimpl{\P}{\M}}$ coincides with $\atom{\ssimplnd{\P}{\M}}$.
If
the theory $\ssimpl{\P}{\M}$ is non-disjunctive,
then
by Lemmata \ref{theo:erasable_nd} and \ref{th:equiv_theories},
the set $\E''=\M\setminus\S$ is an erasable set in $\M$ for $\P$ and the
operator returns $\E'\cup\E''$ (see lines 13-14).

Otherwise,
$\ssimpl{\P}{\M}$ is disjunctive. 
Then, by Theorem \ref{theo:existence} there exists a non-empty
set of atoms $\E''\subseteq(\M\setminus\S)$ such that $\E''$ is super-elementary
for $\ssimpl{\P}{\M}$ and, by Theorem \ref{theo:eliminable_set},
the set
$\E''$ is erasable in $\M\setminus\S$ for $\ssimpl{\P}{\M}$. In this case,
the operator returns the erasable set $\E'\cup\E''$.

As far as point $(ii)$ is concerned,
this is a direct consequence of Theorem \ref{theo:ptime_xiHEF} and this concludes the proof.
\end{proof}

As for minimal model checking, we have the following result.

\begin{theorem}
Given a positive HEF CNF theory $\P$ and a set of atoms
$\N\subseteq\atom{\P}$, checking if $\N$ is a minimal model of $\P$ can be
accomplished in polynomial time.
\end{theorem}
\begin{proof}
The proof follows immediately from Theorem \ref{lemma:LemmaCheck} and Theorem
\ref{theo:GEApolynomial}.
\end{proof}

\begin{exampleContinued}{\ref{ex:pos_cnf}}{Minimal models of positive CNF theories}\rm
Let us consider the execution of 
$\rm GEA_{\xi_{\rm HEF}}(\P,\M)$,
where $\P$
is the HEF theory $\P$ reported in Figure \ref{fig:ex_poscnf}
and $\M=\atom{\P}$.
During the first main iteration,
the eliminating operator $\xi_{\rm HEF}$
returns the $\sel$ set $\{a,b,c,d\}$,
as shown in the example of Section \ref{sect:comp_sel_set}
and $\M$ is set to $\{e$, $f$, $g$, $h$, $i$, $j\}$.
As for the next iteration,
the output of $\xi_{\rm HEF}$ is $\{e$, $f$, $g$, $i\}$ and
$\M$ becomes $\{j,h\}$.
Since now $\M$ coincides with the steady set of 
$\P_\M=\{ j\leftarrow; h\leftarrow; h\leftarrow j; j\leftarrow h\}$,
the algorithm stops returning $\{j,h\}$
as a minimal model of $\P$.
\end{exampleContinued}

\section{Beyond HEF}\label{sect:beyond}

In the previous section, we have shown that GEA($\xi_{\rm HEF}$) computes a
minimal model of a positive HEF CNF theory in polynomial time.
Unfortunately, however, deciding if a given theory 
is head-elementary-free is 
a $\coNP$-complete problem \cite{FassettiP10}\footnote{We note that the reduction therein presented
is still valid for positive HEF CNFs.}.
In other words, while a minimal model for an input HEF CNF theory $\P$ can be indeed computed in polynomial time, checking
whether $\P$ is actually HEF is intractable.

Thus, it is sensible to study the behavior of GEA($\xi_{\rm HEF}$) as applied to a general CNF theory,
which is the subject of this section.
Recall that, by Theorem \ref{theo:ptime_xiHEF}, the \ref{funct:find_elem_nonout_set} function runs in polynomial time
independently of the kind of theory it is applied to.

Next, we will show that
there are non-HEF theories for which GEA($\xi_{\rm HEF}$) successfully returns a minimal model
and others for which GEA($\xi_{\rm HEF}$) ends failing to construct a correct
output\footnote{That the algorithm is not always returning the correct answer is
indeed the expected behavior due to the intractability of the general
problem and since $\xi_{\rm HEF}$ runs in polynomial time (under the assumption that $\Pol \neq \coNP$).} (recall that, on the basis of the results of the
previous section, GEA always returns a correct solution on HEF theories).
The following example should help in clarifying this latter issue.

\begin{myexample}{Behavior on non-HEF theories}\rm
 Consider the following two theories:
\begin{eqnarray*}
\begin{array}{rrll}
{\cal P} = \{& a    & \leftarrow     & \\
         & b,c  & \leftarrow a   & \\
         & c    & \leftarrow b   & \\
         & b    & \leftarrow c   & \}
\end{array}\qquad\qquad
\begin{array}{rrll}
{\cal Q} = \{& a     & \leftarrow     & \\
         & b,c,d & \leftarrow a   & \\
         & c     & \leftarrow b   & \\
         & b     & \leftarrow c   & \\
         & d     & \leftarrow c   & \}
\end{array}
\end{eqnarray*}

Both theories are not HEF. Indeed, the set $\{b,c\}$ is a disjunctive elementary set,
both for ${\cal P}$ and for ${\cal Q}$.
However, while GEA$\rm (\xi_{ HEF})$ does not return a minimal model of ${\cal P}$,
it does correctly compute a minimal model of ${\cal Q}$.

To show that, consider first running GEA$\rm (\xi_{ HEF})$ on ${\cal P}$.
Let $\M$ be $\{a,b,c\}$ (this is the model obtained by taking the union of all the heads).
At line 3 of GEA, $\S$ is set to $\{a\}$, which is not a model of ${\cal P}$ and, then, $\xi_{\rm HEF}$ is invoked.
In particular, the \ref{funct:find_elem_nonout_set} function is executed on the theory
${\cal P}' = \{b,c \leftarrow; c \leftarrow b; b \leftarrow c\}$.
In the execution of the function, $X_0$ is $\{b,c\}$.
Since the elementary graph associated with ${{\cal P}'}^{nd}_{X_0}$
is strongly connected, the function stops and returns $\{b,c\}$.
As a consequence, the set $\E$ is $\{b,c\}$ and the new set $\M$ is $\{a\}$.
It turns out that, since the set $\M$ is not a model for ${\cal P}$ any longer, GEA$\rm (\xi_{ HEF})$
is not able to return a minimal model of ${\cal P}$.
Specifically, at the second iteration of $GEA(\xi_{\rm HEF})$, $\xi_{\rm HEF}$
is invoked on the theory ${\cal P}$ and on the set $\M=\{a\}$.
The steady set $\S$ computed at line 1 in Figure \ref{fig:hef_elim}
is equal to $\{a\}$
{\rm(}since $\Pnd_{\M\leftarrow}$ is the theory $\{a\leftarrow\}${\rm)}
and the theory $\overline{\P}$ is empty.
The set of atoms $\R = \M\setminus\S$ computed at line 3 and
$\overline{\P}_{\R}$ are empty too.
Then the condition at line 9 is true and $\R = \emptyset$ is returned.
Thus, at the second iteration of $GEA(\xi_{\rm HEF})$, $\E$ is empty and, then,
$\S$ is set to $\M = \{a\}$ and returned.
Concluding, $GEA(\xi_{\rm HEF})$ on $\cal P$ ends returning $\{a\}$ which is
not a minimal model of $\P$.

Consider, now, the theory ${\cal Q}$.
Let $\M$ be $\{a,b,c,d\}$, which is the model obtained by taking the union of all the heads.
The set $\S$ is set to $\{a\}$ which is not a model of ${\cal Q}$ and, then, $\xi_{\rm HEF}$ is invoked.
In particular, the \ref{funct:find_elem_nonout_set} function is executed on the theory
${\cal Q}' = \{b,c,d \  {\leftarrow;} \  c\leftarrow b; b \leftarrow c; d\leftarrow c\}$.
In the execution of the function, $X_0$ is $\{b,c,d\}$.
The elementary graph associated with ${{\cal Q}'}^{nd}_{X_0}$ is not strongly connected;
actually, it includes the strongly connected component $C_1$ containing $b$ and $c$ and the strongly connected component $C_2$ containing $d$.
Moreover, there is an edge from $C_1$ to $C_2$ but not vice versa. Then, $C_2$ belongs to the last level of the graph,
and $X_1$ is set to $X_0 \setminus \{d\} = \{b,c\}$.
The elementary subgraph associated with ${{\cal Q}'}^{nd}_{X_1}$ is strongly connected;
therefore the function stops and returns the set $X_1 = \{b,c\}$.
As a consequence the set $\E$ is $\{b,c\}$ and, now, the set $\M$ is $\{a,d\}$ and
the theory ${{\cal Q}'}^{nd}_{\M\leftarrow}$ is $\{a\leftarrow; d\leftarrow a\}$ whose minimal model is $\S = \{a,d\}$.
Since $\S$ is a model of ${\cal Q}$, GEA($\xi_{\rm HEF}$) stops by returning $\S$ as the result, which is indeed a minimal model of ${\cal Q}$.
\end{myexample}

Summarizing, the algorithm GEA($\xi_{\rm HEF}$) always runs in polynomial time and correctly returns a minimal model of HEF CNF theories, but
its correctness on non-HEF theories is seemingly unpredictable: the rest of this section
is devoted to devise a suitable variant of GEA
able to tell about the correctness of the result it returns.
In order to proceed, some further definitions and results are needed.

\begin{definition}[Fallible eliminating operator]
Let $\M$ be a model of a positive CNF theory $\P$.
A \emph{fallible eliminating operator} $\xi_f$
is a polynomial time computable function that
returns a subset of $\M\setminus\S$,
with $\S$ the steady set of $\P$,
with the constraint that if $\xi_f$ returns the empty set, then $\M$ is a minimal model of $\P$.
\end{definition}

\begin{proposition}
Let $\P$ be a positive CNF theory and $\xi_f$ be a fallible eliminating operator.
Checking if the set returned by running GEA($\xi_f$) over $\P$ is a minimal model for $\P$ is attainable in polynomial time.
\end{proposition}
\begin{proof}
By Theorem \ref{theo:gea_correct}, we know that if the set returned by the operator employed in GEA is an erasable set
then the algorithm returns a minimal model. Thus, it is sufficient to check if, at each iteration, $\E$ is an erasable set, namely
it must be checked if $\M\setminus\E$ is a model for $\P$. Since this latter operation can be done in polynomial time,
the statement follows.
\end{proof}

As a consequence of our previous results, we are now able to present the modified GEA,
called the \textit{Incomplete Generalized Elimination Algorithm}
(\emph{IGEA}, for short), which is reported in Figure \ref{fig:incomplete_elimination_algo}.

\begin{figure}[t]
\begin{algorithm}[H]
\footnotesize
\KwIn{A positive CNF theory $\P$}
\KwOut{A minimal model for $\P$ and an indication of a ``{\it success}'' or a model for $\P$ and an 
indication of a ``{\it failure}''}
\BlankLine
$\M = \{h \mid H \leftarrow B \in \P \mbox{ and } h \in H\}$
\tcp{ $\M$ is a (possibly non-minimal) model of $\P$}
$stop$ = $false$\;
\Repeat{$stop$ \label{line:end_out_cycle_inc}}{\label{line:start_out_cycle_inc}
	compute the minimal model $\S$ of $\Pnd_{\M\leftarrow}$\;
	\If{$\S$ is a model of $\P$}{
		$stop$ = $true$\;
	}\Else{
		$\E = \xi_f(\P,\M)$\;
		\If{($\E = \emptyset$)}{
			$\S = \M$\;
			$stop$ = $true$\;
		}\Else{
			\If{($\M\setminus\E$ is not a model of $\P$)}{
				\Return $\M$ and ``Failure''
			}
			$\M = \M \setminus \E$\label{line:compute_xi_S_inc}\;
		}
	}
}
\Return $\S$ and ``Success'' 
\caption{Incomplete Generalized Elimination Algorithm, IGEA($\xi_f$)}
\label{algo:elimination_incomplete}
\end{algorithm}
\caption{Incomplete Generalized Elimination Algorithm, IGEA($\xi_f$)}
\label{fig:incomplete_elimination_algo}
\end{figure}

The following Theorem describes the correctness of IGEA as well as its
computational complexity.

\begin{theorem}\label{theo:igea}
For any fallible eliminating operator $\xi_f$, IGEA{\rm(}$\xi_f${\rm)} always terminates
(with either success or failure)
in polynomial time, returning a model of the input theory.
If it succeeds, then the returned model is a minimal one.
\end{theorem}
\begin{proof}
If the \textit{if} branch at line 12 is never taken, then $\M$ is, at each iteration,
a model for $\P$ and $\E$ is an erasable set. In this case, the fallible eliminating operator $\xi_f$ is indeed an eliminating operator,
whereby IGEA($\xi_f$) behaves as GEA($\xi_f$) does.
This immediately implies that
if IGEA($\xi_f$) does not report a ``failure''
then it returns a minimal model for $\P$.

As far as the time complexity of the algorithm is concerned,
following the same line of reasoning as before,
if the \textit{if} branch at line 12 is never taken,
then IGEA($\xi_f$) requires exactly the same number of iterations as GEA($\xi_f$).
Conversely, if the \textit{if} branch at line 12 is taken, the algorithm ends.
Thus, IGEA($\xi_f$) does not require more iterations than GEA($\xi_f$).
As for the cost of a single iteration,
IGEA($\xi_f$) has only one operation more than GEA($\xi_f$), consisting in checking
if $\M \setminus \E$ is a model for $\P$ (line 12).
Since this operation is the same as that accomplished at line 4,
the asymptotic temporal cost of the algorithm is not affected.
Thus, the cost of IGEA($\xi_f$) is exactly that reported in Proposition \ref{prop:gea_cost} for the GEA,
where $C_{\xi_f}$ is polynomial, by definition of fallible eliminating operator.
\end{proof}

To conclude this section,
we show that $\xi_{\rm HEF}$ can, in fact, be safely adopted as fallible
eliminating operator in IGEA.
The following preliminary proposition is useful.
\begin{proposition}\label{prop:xi_hef_nonempty}
Let $\P$ be CNF theory, $\M$ be a model for $\P$ and $\S$ be the steady set of $\M$ for $\P$.
If $\S$ is not a model for $\P$ then, 
on input $\P$ and $\M$, the operator $\xi_{\rm HEF}$ returns a non-empty set.
\end{proposition}
\begin{proof}
Consider the theory $(\simpl{\P}{\M})_{\M\setminus\S}$.
Since $\S$ is not a model for $\P$, there are rules in $\P$ which are not true in $\S$ but are true in $\M$,
thus that $(\simpl{\P}{\M})_{\M\setminus\S}$ is not empty.
Therefore, it is enough to prove that,
whenever the function \ref{funct:find_elem_nonout_set} is run over a non-empty theory, it  returns a non-empty set.

Consider the function \ref{funct:find_elem_nonout_set} reported in Figure \ref{fig:funct_find_sel}.
First of all, note that if $\P$ is a non-empty theory and $X$ is a non-empty set of atoms occurring in $\P$,
then the elementary graph $\eG(\Pnd_X)$ is non-empty as well.
Thus, in the function, if $X_i$ is non-empty then $\G_i$ is non-empty.

The set $X_0$ at line 1 is non-empty since the function is invoked over a non-empty theory.
By induction, assuming that $X_i$ is non-empty, we prove that $X_{i+1}$ is non-empty as well.

Consider the $(i+1)$-th iteration. Two cases are possible:
\begin{description}
 \item [$(i)$] $\G_i$ is strongly connected and the function ends returning the set $X_i$, which is non-empty by the induction hypothesis.
 \item [$(ii)$] $\G_i$ is not strongly connected and includes at least two connected components.
In such a case, only the atoms of one of the connected components are removed from $X_i$, call it $\C$. Then $X_{i+1} = X_i \setminus \C$ is not empty.
\end{description}

\end{proof}

\begin{proposition}\label{prop:hef_fallible_operator}
 The operator $\xi_{\rm HEF}$ is a fallible eliminating operator.
\end{proposition}
\begin{proof}
Let $\P$ be a general positive CNF theory, $\M$ be a model for $\P$ and $\S$ be the steady set of $\M$ for $\P$.
The proposition is an immediate consequence of the following facts:
$(i)$ $\xi_{\rm HEF}(\P,\M)$ runs in polynomial time (by Theorem \ref{theo:ptime_xiHEF});
$(ii)$
the set returned by the operator
$\xi_{\rm HEF}(\P,\M)$
is a subset of $\M\setminus\S$;
$(iii)$ the set returned by $\xi_{\rm HEF}(\P,\M)$ is not empty (by Proposition \ref{prop:xi_hef_nonempty}).
\end{proof}

Concluding, since $\xi_{\rm HEF}$ is a fallible eliminating operator,
for any CNF theory $\P$, IGEA($\xi_{\rm HEF}$) runs in polynomial time
returning a model and, on HEF theories, we are guaranteed that the returned model is minimal.
Thus, the successful termination of IGEA($\xi_{\rm HEF}$) can be also seen as
a necessary condition for a theory to be HEF
(but it is not a sufficient condition, unless $\coNP$ collapses onto $\Pol$).

The next theorem, finally, summarizes the results of this section.

\begin{theorem}
The algorithm IGEA($\xi_{\rm HEF}$) terminates in polynomial time
for any input positive CNF theory.
Moreover, if the input theory $\P$ is HEF, then {\rm IGEA(}$\xi_{\rm HEF}${\rm)}
succeeds returning a minimal model for $\P$;
otherwise either the algorithm declares success returning a minimal model for $\P$
or the algorithm declares failure returning a model for $\P$.
\end{theorem}

\begin{proof}
The proof immediately follows from Theorem \ref{theo:igea} and Proposition \ref{prop:hef_fallible_operator}.
\end{proof}

\section{Conclusions}\label{sect:conclusions}

Tasks related with computing with minimal models are relevant to several AI
applications.

The focus of this paper has been devising efficient algorithms to deal with
minimal models of CNF theories. Particularly, three problems have been mainly
considered, that are, minimal model checking, minimal model finding and model
minimization. All these problems prove themselves to be intractable for general
CNF theories, while it was known that they become tractable for the class of
head-cycle-free theories \cite{Ben-Eliyahu-ZoharyP97} and, in fact, to the best
of our knowledge, positive HCF theories form the largest class of CNFs
for which polynomial time algorithms solving minimal model finding and minimal
model checking 
are known so far. 
The research presented here follows the same research target as that of
\cite{Ben-Eliyahu-ZoharyP97} and
the main contribution of this work
has been that of designing a polynomial time algorithm for computing a minimal
model for (a superset of) the class of positive
HEF (Head Elementary-Set Free) CNF theories, a strict superset of the class of
HCF theories, whose definition naturally stems for the analogous one given in
the context of logic programming in
\cite{GebserLL06}. This contribution thus broadens the tractability frontier
associated with  minimal model computation problems.

In more detail, we have introduced the {\em Generalized Elimination Algorithm
(GEA)}, that computes a minimal model of {\em any} positive CNF,
whose complexity depends on the
complexity of the specific {\em eliminating operator} it invokes. Therefore, in
order to attain tractability. a specific eliminating
operator has been defined which allows for the algorithm to compute in
polynomial time a minimal
model for a class of CNF that strictly includes HEF
theories.

However, it is unfortunately already known that recognizing HEF theories is
``per s\'e'' an intractable problem, which may apparently limit the
applicability range of our algorithmic schema. In order to overcome such a
problem, an ``incomplete'' variant of the GEA (called IGEA) is proposed: the
resulting schema, once instantiated with an appropriate
elimination operator, always constructs a model of the input CNF, which  is
guaranteed to be minimal at least if the input theory is HEF. We note that this
latter algorithm is able to ``declare'' if the returned model is indeed minimal
or not.

The research work presented here can be continued along some interesting
direction. As a major research direction, since the IGEA is capable to deal also
with theories
that are not HEF, it would be relevant to define, via a syntactic specification,
as those pinpointing HCF and HEF theories, a superset HEF theories coinciding
with those on which the IGEA stops returning a success. While it is not at all
clear if this can be reasonably attained, we might consider it enough to get
close (from below) to this class of theories. 
Very related to the above line of research, there is the assessment
of the practical occurrence of theories having the HEF property
or the property of guaranteeing success to the IGEA
and also the assessment of the  
success rate of the IGEA on generic CNF theories.
Moreover, enhancing stable models and answer set engines for
logic programs with the IGEA 
appears a potentially fruitful line of investigation.

\bibliographystyle{plain}

\section*{Appendix}

\begin{proofOf}{Lemma \ref{th:equiv_theories}}

In order to prove the theorem, we have to show that
there is no clause in
$\P$ which is false in $\M\setminus\E$
if and only if there is no clause in
$\ssimpl{\P}{\M}$ which is false in $(\M\setminus\S)\setminus\E$.

The proof is organized in two parts, both proved by contradiction.

\medskip
($\Longrightarrow$)
First of all note that, for each clause
$c\equiv H\leftarrow B$ in $\ssimpl{\P}{\M}$
there is, by definition, a clause $c'\equiv H'\leftarrow B'$
in $\simpl{\P}{\M}$ such that
$H'_{\M\setminus\S}=H\neq\emptyset$ and $B'_{\M\setminus\S}=B$.
Moreover, note that (by construction of $\simpl{\P}{\M}$)
$c'$ is also in $\P$ and that $B' \subseteq \M$.

For the sake of contradiction, assume that $c$ is true in $\M\setminus\S$ and
false in $(\M\setminus\S)\setminus\E$.
We aim at proving that if such a rule exists then there exists at least one
rule (actually $c'$) which is false in $\M\setminus\E$.

If $c$ is false in $(\M\setminus\S)\setminus\E$, then $B$ is contained in
$(\M\setminus\S)\setminus\E$ and $\E$ contains $H$.
Note that since $B\cap\E = \emptyset$ and since $B=B'\cap(\M\setminus\S)$,
it follows that also $B'\cap\E = \emptyset$ (and, then, $B'\subseteq
\M\setminus\E$).

But, since $H'\cap\S = \emptyset$ (by construction of $\simpl{\P}{\M}$) and
since $H = H'\cap(\M\setminus\S)$, if $H\subseteq \E$ then also $H'\subseteq
\E$.

So, we have proved that $B'\subseteq (\M\setminus\E)$ and $H'\subseteq\E$; it
follows that $c'$ is
false in $\M\setminus\E$, which concludes the first part of the proof.

\medskip
($\Longleftarrow$)
For the sake of contradiction, assume that
there exists a clause $c'\equiv H\leftarrow B$ in $\P$ which
is true in $\M$ and false in $\M\setminus\E$.
Thus, $B\subseteq(\M\setminus\E)$ and $\emptyset\subset H\subseteq\E$
implying that $H\cap\S=\emptyset$ since, by definition of steady set and
erasable set, any erasable set has empty intersection with the steady set,
namely $\E\cap\S = \emptyset$.

Consider, now, the clause $c\equiv H_{\M\setminus\S}\leftarrow
B_{\M\setminus\S}$.
Since $H$ is contained in $\E$ and $\E\cap\S = \emptyset$, it follows that
$H_{\M\setminus\S}$ is equal to $H$ and is not empty, then $c$ is by
definition in $\ssimpl{\P}{\M}$.
Conversely, since $B\subseteq(\M\setminus\E)$ then
$B_{\M\setminus\S}\cap\E=\emptyset$.

Thus, $c$ is a clause of $\ssimpl{\P}{\M}$ whose head is included in $\E$ meaning
that $c$ is false in $(\M\setminus\S)\setminus\E$. But, since $\M\setminus\S$ is
a model of $\ssimpl{\P}{\M}$, $c$ is true in $\M\setminus\S$.
This concludes the proof since it contradicts that $\E$ is an erasable
set in $\M\setminus\S$ for $\ssimpl{\P}{\M}$.

\end{proofOf}

\begin{proofOf}{Lemma \ref{prop:equiv_th_atom}}
For the sake of contradiction, assume that such a clause
$c\equiv h\leftarrow$ belongs to $\ssimpl{\P}{\M}$. Clearly,
in this case there exists at least one clause $c'$ in
$\simpl{\P}{\M} \subseteq \P$ such that $c = c'_{\M\setminus\S}$.
By definition of $\ssimpl{\P}{\M}$, it is the case
that $h$ is the only atom in $\M$ (and also in $\M\setminus\S$)
occurring in the head of the clause $c'$ and, hence,
it can be concluded that $c'_{\M\leftarrow}$
belongs also to the theory $\Pnd_{\M\leftarrow}$.
The body of $c'$ is contained in $\M$ by definition of $\simpl{\P}{\M}$.
Since the body of $c$ is empty, there are two possibilities:
\begin{enumerate}
\item
The body of $c'$ is also empty: but in this
case $h$ should belong to $\S$ and the clause $h\leftarrow$
cannot be in $\ssimpl{\P}{\M}$, a contradiction;

\item
The body of $c'$ is contained in $\S$:
but this means that $c'_{\M\leftarrow}$ is a clause of $\Pnd_{\M\leftarrow}$
which is false in $\S$, which contradicts the fact that
$\S$ is the minimal model of $\Pnd_{\M\leftarrow}$.
\end{enumerate}
\end{proofOf}

\begin{proofOf}{Lemma \ref{theo:HEF_monotonicity}}
The proof is given by contraposition:
assuming that $\P'_X$ is not HEF it is derived that $\P$ is not HEF.
If $\P'_X$ is not HEF then, by Proposition \ref{prop:hef_def},
there is a set of atoms $E\subseteq X$ such that $E$
is both a disjunctive and an elementary set for $\P'_X$.
Clear enough, if $E$ is a disjunctive set for $\P'_X$ then it is a disjunctive set for $\P$ as well.
Moreover, if $E$ is elementary for $\P'_X$
then, for each proper subset $O\subset E$, there is a clause $c_X\equiv H_X
\leftarrow B_X$ in $\P'_X$
such that $H_X \cap O \neq \emptyset$, $H_X\cap (E\setminus O) = \emptyset$, $B_X \cap O = \emptyset$ and $B_X\cap (E\setminus O) \neq \emptyset$.
By definition of $\P'_X$, $c_X \in \P'_X$ implies that there is
a clause $c\equiv H \leftarrow B$ in $\P$
such that $H_X = H \cap X$ and $B_X = B \cap X$.
Thus, since
$O\subset E\subseteq X$,
it follows that
$H \cap O = H_X \cap O \neq \emptyset$,  $H \cap (E\setminus O) = H_X\cap (E\setminus O) = \emptyset$,
$B \cap O = B_X \cap O = \emptyset$ and $B \cap (E\setminus O) = B_X\cap (E\setminus O) \neq \emptyset$.
Therefore, $O$ is outbound in $E$ also for $\P$.
As a consequence, $E$ is elementary for $\P$
and, since $E$ is also a disjunctive set for $\P$,
it is the case that $\P$ is not HEF.
\end{proofOf}

\begin{proofOf}{Lemma \ref{theo:pnd->p}}

Before stating the Lemma, it is needed to recall a result given in
\cite{FassettiP10}
asserting that if
$O\subseteq\atom{\Pnd}$ is elementary for $\Pnd$ then it is elementary also for
$\P$.

\begin{claim}[Rephrased from \cite{FassettiP10}]\label{lemma:invariability}
Let $\P$ be a CNF theory and $\P'\subseteq \P$ any CNF
consisting of a subset of the clauses of $\P$.
If $O\subseteq \atom{\P'}$ is an elementary set for $\P'$, then $O$ is
an elementary set for $\P$ as well.
\end{claim}

By hypothesis, $O$ is non-outbound in $\atom{\Pnd}$ for $\Pnd$.
In order to complete the proof, it is enough to prove that, since $\P$ is HEF, $O$ is non-outbound in $\atom{\P}$ for $\P$, which is shown next.

By contradiction, assume that $O$ is outbound in $\atom{\P}$ for $\P$.
Since $\atom{\P}$ is the set of all the atoms appearing in $\P$,
then there exists a clause
$H \leftarrow B$ in $\P$
such that $B\subseteq \atom{\P}\setminus O$ and $H \subseteq O$.
Since $O$ is non-outbound in $\atom{\Pnd}$ for $\Pnd$,
then the clause $H\leftarrow B$ is not in $\Pnd$ and it holds that $|H| \ge 2$.
As a consequence, $O$ is an elementary set for $\P$ and there
is a clause $H \leftarrow B$ such that $|H\cap O|\ge 2$.
That is to say, $\P$ is not HEF, a contradiction.
\end{proofOf}

\begin{proofOf}{Lemma \ref{theo:minimal->elem}}
Let $O$ be a minimal non-outbound set in $\atom{\P}$ for $\P$.
Hence, there is no single head clause $h \leftarrow B$ in $\P$
such that $B \subseteq \atom{\P}\setminus O$ and $h \in O$.

Consider, now, any non-empty proper subset $O'$ of $O$.
By hypothesis of minimality of $O$, the set $O'$ is outbound in $\atom{\P}$
for $\P$ and,
hence, there exists a (single head) clause $h \leftarrow B$ in $\P$
such that $B \subseteq \atom{\P}\setminus O'$ and $h \in O'$.
Moreover,
notice that it is the case that the body of such a clause
has non-empty intersection with the set $O$, for otherwise the set $O$ would be outbound.

Thus, it holds that $B \cap O \neq \emptyset$ and, since $B\cap O'=\emptyset$,
it also holds that
$B \cap (O \setminus O') \neq \emptyset$.

It can be therefore concluded that $h\leftarrow B$ is a clause
such that $B \cap O' = \emptyset$,
$B \cap (O\setminus O') \neq \emptyset$ and $h \in O'$ holds.
The existence of such a clause  implies that $O'$ is outbound in $O$ for $\P$.

Since any proper subset of $O$ is outbound in $O$ for $\P$, $O$ is also elementary for $\P$
and, hence, super-elementary for $\P$.
\end{proofOf}

\begin{proofOf}{Lemma \ref{th:worthless_atoms}}
The proof is an immediate consequence of
Lemma \ref{th:equiv_theories}.
\end{proofOf}

\begin{proofOf}{Lemma \ref{theo:singleton_erasable}}
Let $a$ be an atom occurring in $\M$
but not in $\atom{\Pnd_{\M\leftarrow}}$.
Since $a$ does not occur in $\atom{\Pnd_{\M\leftarrow}}$, two cases
have to be taken into account, that are:
($i$) $a$ occurs only in the body of some disjunctive rule in $\P_{\M\leftarrow}$,
and ($ii$) $a$ occurs in the head of some disjunctive rule in $\P_{\M\leftarrow}$.

In the first  case, the set $\M\setminus\{a\}$ is a model of $\P$,
since the only effect of removing $a$ from $\M$ is to falsify the body
of some rule of $\P$.

In the second  case, the head of no rule $H\leftarrow B$ in $\P$ can be
falsified since if $\{a\}\subset H$ holds, then $|\M\cap H|\ge 2$ holds.
Thus, $\M\setminus\{a\}$ is a model for $\P$.
\end{proofOf}

\begin{proofOf}{Lemma \ref{theo:erasable_nd}}
The proof follows immediately from Lemma \ref{prop:equiv_th_atom}.
\end{proofOf}

\end{document}